\documentclass[letterpaper]{article} 
\usepackage{aaai25}  
\usepackage{times}  
\usepackage{helvet}  
\usepackage{courier}  
\usepackage[hyphens]{url}  
\usepackage{graphicx} 
\usepackage{natbib}  
\usepackage{caption} 
\usepackage{algorithmic}

\usepackage{newfloat}
\usepackage{listings}
\DeclareCaptionStyle{ruled}{labelfont=normalfont,labelsep=colon,strut=off} 
\usepackage{amsmath}
\usepackage{amsfonts}
\usepackage{mathrsfs}
\usepackage{amssymb}
\usepackage{amsthm}

\usepackage[linesnumbered, ruled,vlined]{algorithm2e}

\usepackage{subcaption}

\newtheorem{definition}{Definition}

\usepackage{thmtools}
\usepackage{thm-restate}

\usepackage{pst-node}
\usepackage{amsmath}
\usepackage{cleveref}
\usepackage{pgfplotstable}

\usepackage{pgfgantt}
\usepackage{xcolor,cancel}
\usepackage{soul}

\usepackage{xspace}
\def\ie{{\em i.e.,}\xspace}
\def\eg{{\em e.g.,}\xspace}
\def\cf{{\em cf.}\xspace}
\def\wrt{{\em w.r.t.}\xspace}

\DeclareMathOperator*{\argmax}{\arg\,\max}

\usepackage{pgfplots}
\usepackage{tcolorbox}
\usepackage{colortbl}
\usepackage{booktabs,multirow}
\usepackage{eurosym}
\usepackage{xspace}
\usepackage{tikz}
\usetikzlibrary{shapes}
\usetikzlibrary{arrows}
\usetikzlibrary{automata}
\usetikzlibrary{fit}
\usetikzlibrary{shadows}
\usetikzlibrary{shapes}
\usetikzlibrary{trees}
\usetikzlibrary{snakes}
\usetikzlibrary{positioning}
\usetikzlibrary{patterns}
\usetikzlibrary{backgrounds}
\usetikzlibrary{matrix}
\usetikzlibrary{pgfplots.groupplots}
\usetikzlibrary{external}

\definecolor{mygreylight}{rgb}{0.7,0.7,0.7}
\definecolor{mygreyverylight}{rgb}{0.9,0.9,0.9}
\definecolor{mybluelight}{rgb}{ 0.7137,0.8666,0.9098}
\definecolor{mybluedark}{rgb}{ 0.1921,0.5176,0.6078}
\definecolor{mygreen}{rgb}{ 0.1,0.8,0.1}
\definecolor{cwtitleblue}{rgb}{0.2,0.4,0.6}
\definecolor{cwblue1}{rgb}{0.27,0.427,0.623}
\definecolor{cwblue2}{rgb}{0.286,0.454,0.658}
\definecolor{cwblue3}{rgb}{0.6,0.8,0.99}
\definecolor{wp1color}{rgb}{0.64, 0.76, 0.68} 
\definecolor{wp1tcolor}{rgb}{0.74, 0.86, 0.78} 
\definecolor{wp2color}{rgb}{1.0, 0.99, 0.82} 
\definecolor{wp2tcolor}{rgb}{1.0, 0.99, 0.92} 
\definecolor{wp3color}{rgb}{0.74, 0.72, 0.42} 
\definecolor{wp4color}{rgb}{0.52, 0.73, 0.4} 
\definecolor{ganttdrawcolor}{rgb}{0.64, 0.76, 0.68} 

	\definecolor{sthlmLightBlue}{RGB}{214,237,252} 
		
	\definecolor{sthlmBlue}{RGB}{0,110,191} 

	\definecolor{sthlmLightGreen}{RGB}{213,247,244} 
		
	\definecolor{sthlmGreen}{RGB}{0,134,127} 

	\definecolor{sthlmLightGrey}{RGB}{213,217,225} 
		
	\definecolor{sthlmGrey}{RGB}{245,243,238} 
		
	\definecolor{sthlmDarkGrey}{RGB}{51,51,51} 

	\definecolor{sthlmLightOrange}{RGB}{255,215,210} 
		
	\definecolor{sthlmOrange}{RGB}{221,74,44} 

	\definecolor{sthlmLightPurple}{RGB}{241,230,252} 
		
	\definecolor{sthlmPurple}{RGB}{93,35,125} 

	\definecolor{sthlmLightRed}{RGB}{254,222,237} 
		
	\definecolor{sthlmRed}{RGB}{196,0,100} 

	\definecolor{sthlmYellow}{RGB}{252,191,10} 

\usepackage{xcolor}
\definecolor{pink}{rgb}{0.858, 0.188, 0.478}

\newcommand{\persComment}[3]{
  \ifmmode
  \text{\textcolor{#3}{[#2] #1}}
  \else
  \textcolor{#3}{[#2] \em #1}
  \fi
}

\newcommand{\draftVersion}{x}
\ifdefined\draftVersion

\newcommand{\Jilles}[1]{\persComment{#1}{jsd}{violet}}
\newcommand{\Johan}[1]{\persComment{#1}{jp}{blue}}
\newcommand{\Rafael}[1]{\persComment{#1}{rc}{brown}}
\newcommand{\Aurelien}[1]{\persComment{#1}{ad}{pink}}

\else

\newcommand{\Jilles}[1]{} 
\newcommand{\Rafael}[1]{} 
\newcommand{\Johan}[1]{} 
\newcommand{\Aurelien}[1]{} 

\fi

\usepackage{colortbl}%
  \newcommand{\myrowcolour}{\rowcolor[gray]{0.925}}
\newcommand{\highest}[1]{\textcolor{sthlmRed}{\textbf{#1}}}

\title{Optimally Solving Simultaneous-Move Dec-POMDPs: \\ The Sequential Central Planning Approach}
\author {
    Johan Peralez\textsuperscript{\rm 1},
    Aurelien Delage\textsuperscript{\rm 1},
    Jacopo Castellini\textsuperscript{\rm 2},
    Rafael F. Cunha\textsuperscript{\rm 3},
    Jilles S. Dibangoye\textsuperscript{\rm 3}
}
\makeatletter
\affiliations {
    \textsuperscript{\rm 1}Univ Lyon, INSA Lyon, Inria, CITI, EA3720, 69621 Villeurbanne, France\\
    \textsuperscript{\rm 2}Haute Ecole de Gestion de Genève,
University of Applied Sciences and Arts Western, 1227 Carouge, Geneva, Switzerland\\
    \textsuperscript{\rm 3}Bernoulli Institute, University of Groningen, Nijenborgh 4, NL-9747AG, Groningen, Netherlands\\
johan.peralez@gmail.com, aurelien.delage@insa-lyon.fr, jacopo.castellini@hesge.ch, \{r.f.cunha,j.s.dibangoye\}@rug.nl
}
\makeatother

\begin{document}

\maketitle

\begin{abstract}
The centralized training for decentralized execution paradigm emerged as the state-of-the-art approach to $\epsilon$-optimally solving decentralized partially observable Markov decision processes. However, scalability remains a significant issue.
This paper presents a novel and more scalable alternative, namely the sequential-move centralized training for decentralized execution.
This paradigm further pushes the applicability of the \citeauthor{bellman}'s principle of optimality, raising three new properties. 
First, it allows a central planner to reason upon sufficient sequential-move statistics instead of prior simultaneous-move ones.
%
Next, it proves that $\epsilon$-optimal value functions are piecewise linear and convex in such sufficient sequential-move statistics.
Finally, it drops the complexity of the backup operators from double exponential to polynomial at the expense of longer planning horizons.  
%
Besides, it makes it easy to use single-agent methods, \eg SARSA algorithm enhanced with these findings, while still preserving convergence guarantees.
Experiments on two- as well as many-agent domains from the literature against $\epsilon$-optimal simultaneous-move solvers confirm the superiority of our novel approach.
This paradigm opens the door for efficient planning and reinforcement learning methods for multi-agent systems.
\end{abstract}

\section{Introduction}
\label{sec:introdiction}

The problem of designing a system with multiple agents that cooperate to control a hidden Markov chain is of interest in artificial intelligence, optimal decentralized and stochastic control, and game theory \citep{Yoshikawa1978623,radner1962,Ooi96,639686,Shoham:2008:MSA:1483085}. 
Simultaneous-move decentralized partially observable Markov decision processes (Dec-POMDPs) emerged as the standard framework for formalizing and solving such problems when agents act simultaneously. 
Given the growth of multi-agent intelligence systems, \eg online interactive services, the Internet of Things, and intelligent transportation systems, the importance of scalable multi-agent solutions is clear.
It also highlights the impetus for methods with theoretical guarantees to find safe and secure solutions.
However, as the situation currently stands, $\epsilon$-optimally solving Dec-POMDPs is widely believed to be out of reach: finite-horizon cases are NEXP-hard, infinite-horizon cases are undecidable \citep{639686}, and approximations remain intractable \citep{860816}. 

There are two distinct but interdependent explanations for the limited scalability of existing algorithms. 
The more widely known reason is that agents in such a setting can neither see the state of the world nor explicitly communicate with one another due to communication costs, latency, or noise. 
However, what one agent sees and does directly affects what the others see and do.
That \textbf{silent coordination problem} makes it hard to define a right notion of state necessary to apply dynamic programming theory \citep{HansenBZ04}. 
A standard solution method suggests reducing the original multi-agent problem into a single-agent one, which is then solved using single-agent algorithms based on recent advances in Markov decision processes.
This simultaneous-move centralized training for decentralized execution paradigm applies successfully to planning \citep{Szer05,OliehoekSV08,OliehoekSAW13,Dibangoye:OMDP:2016,sokota2021solving}
and reinforcement learning \citep{FoersterFANW18,RashidSWFFW18,BonoDMP018,DibangoyeBICML18}. 
The less well-known reason for the poor scaling behavior is the requirement of current methods for simultaneous-move decision-making.
In such methods, the simultaneous-move greedy selection of joint decision rules, \ie simultaneous-move \citeauthor{bellman}'s backup, is performed collectively for agents---all at once.
Unfortunately,  this \textbf{decision entanglement problem} is double exponential with agents and time, rendering even a single simultaneous-move \citeauthor{bellman}'s backup prohibitively expensive when facing realistic problems with long planning horizons---let alone finding an optimal solution.

In contrast, this paper builds upon the assumption that decision rules to be executed simultaneously can nonetheless be selected sequentially. 
A paradigm we refer to as sequential-move centralized training for simultaneous-move decentralized execution.
This paradigm pushes further the applicability of \citeauthor{bellman}'s principle of optimality so that a central planner can still select joint decision rules, but now one private decision rule at a time.
Doing so disentangles the decision variables with mutual influence and raises three new properties. 
First, it allows a sequential-move central planner to reason upon sufficient sequential-move statistics instead of prior simultaneous-move ones.
Next, it proves that $\epsilon$-optimal value functions are still piecewise linear and convex in sufficient sequential-move statistics.
Finally, it drops the complexity of \citeauthor{bellman}'s backups from double exponential to polynomial at the expense of longer planning horizons.
We show that a SARSA algorithm enhanced by these findings applies while preserving its convergence guarantees. 
Experiments conducted in standard two- and multi-agent domains from existing literature, particularly against $\epsilon$-optimal solvers such as feature-based heuristic search value iteration and local methods \citep{Tan:1997:MRL:284860.284934,Konda:NIPS00}, demonstrate the clear advantages of our sequential-move variant of the oSARSA algorithm \citep{DibangoyeBICML18}.

This remainder of this paper is organized as follows. Section 2 examines the dominant paradigm for optimally solving simultaneous-move Dec-POMDPs: simultaneous-move centralized training for decentralized execution. Although this paradigm has demonstrated efficiency in small to medium-sized teams of cooperative agents, it faces scalability issues when applied to larger teams. In Section 3, we advocate for a paradigm shift from simultaneous-move to sequential-move centralized training for decentralized execution. We establish that simultaneous-move Dec-POMDPs can be reduced to sequential-move ones, enabling knowledge transfer from the former to the latter. Section 4 illustrates that all existing structural results from the dominant paradigm are transferable to this new approach. However, this shift also uncovers an exponential drop in the complexity of backup operators. Finally, Section 5 adapts the oSARSA algorithm from the original paradigm to the new framework. The experiments in Section 6 provide substantial evidence for the superiority of the proposed paradigm.

\section{Simultaneous-Move Central Planning}
\label{sec:background}

This section presents an overview of the dominant paradigm for solving cooperative multi-agent problems: simultaneous-move centralized training for decentralized control. We start by distinguishing between the multi-agent and single-agent formulations of simultaneous-move Dec-POMDPs. The multi-agent framework, while being the original design, faces significant challenges in dynamic programming applications. In contrast, the single-agent formulation enables a central planner to orchestrate the actions of all agents concurrently, thereby illustrating simultaneous-move centralized training for decentralized control. This paradigm leverages single-agent theories and algorithms and has emerged as the dominant approach for (optimally) solving simultaneous-move Dec-POMDPs. 


\subsection{Simultaneous-Move Multi-Agent Formulation}
\begin{definition}
\label{def:decpomdp}
A $n$-agent simultaneous-move Dec-POMDP is given by tuple $\mathtt{M} \doteq (\mathtt{X}, \mathtt{U}, \mathtt{Z}, \mathtt{p},\mathtt{r})$. , where: $\mathtt{X}$ is a finite set of states, denoted $\mathtt{x}$ or $\mathtt{y}$;  $\mathtt{U} \doteq \mathtt{U}^1\times \cdots \times \mathtt{U}^n$ is a finite set of actions, denoted $\mathtt{u}=(\mathtt{u}^1,\cdots,\mathtt{u}^n)$;  $\mathtt{Z} \doteq \mathtt{Z}^1\times \cdots \times \mathtt{Z}^n$ is a finite set of observations, denoted $\mathtt{z}=(\mathtt{z}^1,\cdots,\mathtt{z}^n)$;  $\mathtt{p}$ is the transition function that specifies the probability  $\mathtt{p}_{\mathtt{x},\mathtt{y}}^{\mathtt{u},\mathtt{z}}$ of the process state being $\mathtt{y}$ and having observation $\mathtt{z}$ after taking action $\mathtt{u}$ in state $\mathtt{x}$;  and
 $\mathtt{r}$ is the reward function specifying the reward $\mathtt{r}_{\mathtt{x},\mathtt{u}}$ received after taking action $\mathtt{u}$ in state $\mathtt{x}$.
\end{definition}

Throughout the paper, we make the following assumptions: (1) rewards are two-side bounded, \ie there exists some $c>0$, such that $\|\mathtt{r}_{\cdot,\cdot}\|_\infty \leq c$;
and planning horizon $\ell$ is finite. This is commonly achieved in infinite-horizon problems with a discount factor $\gamma \in [0,1)$ by searching for an $\epsilon$-optimal solution (for some $\epsilon>0$) and setting $\ell = \lceil \log_\gamma{(1-\gamma)\epsilon}/{c} \rceil$.
Optimally solving $\mathtt{M}$ aims at finding (simultaneous-move) joint policy $\xi  = (\mathtt{a}^1_{0:\ell-1}, \cdots, \mathtt{a}^{n}_{0:\ell-1})$, \ie an $n$-tuple of sequences of private decision rules, 
one sequence of private decision rules $\mathtt{a}^i_{0:\ell-1} = (\mathtt{a}^i_0,\ldots,\mathtt{a}^i_{\ell-1})$ for each agent $i$. 
An optimal joint policy maximizes the expected $\gamma$-discounted cumulative reward starting from initial state distribution $b_0$ onward. This value can be expressed as: 
%
\begin{align*}
\upsilon_{\mathtt{M},\gamma,0}^\xi(b_0) \doteq \mathbb{E}\{ \textstyle{\sum_{t=0}^{\ell-1}}~\gamma^t \mathtt{r}_{\mathtt{x}_t,\mathtt{u}_t} ~|~b_0, \xi\}.
\end{align*}
Let $\mathbb{N}_{\leq m} \doteq \{0, 1,\ldots, m\}$ and $\mathbb{N}^*_{\leq m} \doteq \mathbb{N}_{\leq m}\backslash \{0\}$.
For each agent $i\in \mathbb{N}^*_{\leq n}$, private decision rule $\mathtt{a}^i_t\colon \mathtt{o}^i_t \mapsto \mathtt{u}^i_t$ depends on $t$-th private history $\mathtt{o}^i_t \doteq (\mathtt{u}^i_{0:t-1},\mathtt{z}^i_{1:t})$, with $0$-th private history being $\mathtt{o}^i_0 \doteq \emptyset$.
At each step $t\in \mathbb{N}_{\leq \ell-1}$, we denote $\mathtt{O}^i_t$ and $\mathtt{A}^i_t$  the finite set of agent $i$' $t$-th private histories $\mathtt{o}^i_t\in \mathtt{O}^i_t$  and decision rules $\mathtt{a}^i_t\in \mathtt{A}^i_t$, respectively.

Optimally solving $\mathtt{M}$ using dynamic programming theory in its simultaneous-move multi-agent formulation is non-trivial, since it is not clear how to define a right notion of state \citep{HansenBZ04}.
To better understand this, notice that every agent acts simultaneously, but can neither see the actual state of the world nor explicitly communicate its actions and observations with others.
At the same time, what one agent sees and does directly affects what the others see and do, which explains the mutual influence of all decision variables $\mathtt{a}_t = (\mathtt{a}^1_t, \cdots, \mathtt{a}^{n}_t)$ at each time step $t\in \mathbb{N}_{\leq \ell-1}$, \ie a mapping called $t$-th joint decision rule $\mathtt{a}_t\colon \mathtt{o}_t \mapsto \mathtt{u}_t$ from joint histories $\mathtt{o}_t \doteq (\mathtt{o}^1_t,\ldots,\mathtt{o}^{n}_t)$ to joint actions.
At each step $t\in \mathbb{N}_{\leq \ell-1}$, we denote $\mathtt{O}_t$ and $\mathtt{A}_t$  the finite set of $t$-th joint histories $\mathtt{o}_t\in \mathtt{O}_t$ and joint decision rules $\mathtt{a}_t\in \mathtt{A}_t$, respectively.
The motivation for a single-agent reformulation is two-fold. 
The primary reason is that it allows us to reason simultaneously about all mutually dependent decision variables $\mathtt{a}_t= (\mathtt{a}^1_t, \cdots, \mathtt{a}^{n}_t)$, \ie a $t$-th joint decision rule.
Besides, the single-agent reformulation eases the transfer of theory and algorithms from single- to multi-agent systems.

\subsection{Single-Agent Reformulation}

This equivalent single-agent reformulation aims at recasting $\mathtt{M}$ from the perspective of an offline simultaneous-move central planner.
At every step, this planner simultaneously acts on behalf of all agents, taking a joint decision rule, but receives no feedback in the worst-case scenario. 
A statistic of the history of selected joint decision rules $\mathtt{a}_{0:t-1} = (\mathtt{a}_0,\ldots,\mathtt{a}_{t-1})$, \ie $t$-th simultaneous-move occupancy state $\mathtt{s}_t \doteq (\Pr\{\mathtt{x}_t,\mathtt{o}_t\mid b_0, \mathtt{a}_{0:t-1}\})_{\mathtt{x}_t\in \mathtt{X},\mathtt{o}_t\in \mathtt{O}_t}$, is a conditional probability distribution over states and $t$-th joint histories. It describes a non-observable and deterministic Markov decision process, namely a simultaneous-move occupancy-state MDP ($o$MDP) \citep{Dibangoye:OMDP:2016}.

\begin{definition}
A simultaneous-move $o$MDP \wrt $\mathtt{M}$ is given by a tuple $\mathtt{M}' \doteq ( \mathtt{S}, \mathtt{A}, \mathtt{T}, \mathtt{R} )$, where 
$\mathtt{S} = \cup_{t\in \mathbb{N}_{\leq \ell-1}} \mathtt{S}_t$ is a collection of simultaneous-move occupancy states, denoted $\mathtt{s}_t\in \mathtt{S}_t$;
$\mathtt{A}$ is the space of actions describing joint decision rules;
$\mathtt{T}\colon \mathtt{S} \times \mathtt{A} \to \mathtt{S}$ is the deterministic transition rule, where next occupancy state is given by $\mathtt{s}_{t+1} \doteq \mathtt{T}(\mathtt{s}_t,\mathtt{a}_t)$, 
\begin{align*}
\mathtt{s}_{t+1} \colon (\mathtt{y}, (\mathtt{o},\mathtt{u},\mathtt{z})) & \mapsto {\textstyle \sum_{\mathtt{x}\in \mathtt{X}}}~ \mathtt{s}_t(\mathtt{x},\mathtt{o}) \cdot \mathtt{p}_{\mathtt{x},\mathtt{y}}^{\mathtt{a}_t(\mathtt{o}) ,\mathtt{z}},
\end{align*}
and $\mathtt{R}\colon \mathtt{S} \times \mathtt{A} \to \mathbb{R}$ describes the linear reward function---\ie
\begin{align*}
\mathtt{R}\colon (\mathtt{s}_t,\mathtt{a}_t) & \mapsto {\textstyle\sum_{\mathtt{x}\in \mathtt{X}}\sum_{\mathtt{o}\in \mathtt{O}_t}}~ \mathtt{s}_t(\mathtt{x},\mathtt{o}) \cdot \mathtt{r}_{\mathtt{x},\mathtt{a}_t(\mathtt{o})}.
\end{align*}
\end{definition}

Optimally solving $\mathtt{M}'$ aims at finding a simultaneous-move joint policy $\xi$ via an optimal value function $(\upsilon_{\mathtt{M}',\gamma,t}^*)_{t\in \mathbb{N}_{\leq \ell}}$, \ie mappings from states to reals, 
$
\upsilon_{\mathtt{M}',\gamma,t}^*(\mathtt{s}_t) \doteq \max_{\mathtt{a}_{t:\ell-1}\in \mathtt{A}_{t:\ell-1}}~\upsilon_{\mathtt{M}',\gamma,t}^{\mathtt{a}_{t:\ell-1}}(\mathtt{s}_t),
$
 and 
solution of \citeauthor{bellman}'s optimality equations, \ie
\begin{align}
\upsilon_{\mathtt{M}',\gamma,t}^*(\mathtt{s}_t) &= {\textstyle\max_{\mathtt{a}_t\in \mathtt{A}_t} ~\mathtt{q}_{\mathtt{M}',\gamma,t}^*(\mathtt{s}_t,\mathtt{a}_t)}, \label{eq:bellman}\\
\mathtt{q}_{\mathtt{M}',\gamma,t}^*(\mathtt{s}_t,\mathtt{a}_t) &\doteq \mathtt{R}(\mathtt{s}_t,\mathtt{a}_t) + \gamma \upsilon_{\mathtt{M}',\gamma,t+1}^*(\mathtt{T}(\mathtt{s}_t,\mathtt{a}_t)),\label{eq:q:bellman}
\end{align} 
with boundary condition $\upsilon_{\mathtt{M}',\gamma,\ell}^*(\cdot)=\mathtt{q}_{\mathtt{M}',\gamma,\ell}^*(\cdot,\cdot)=0$. 
One can greedily select an optimal joint policy given the optimal value function. 
Unfortunately, solving \citeauthor{bellman}'s optimality equations (\ref{eq:bellman}) is not trivial since simultaneous-move occupancy-state space $\mathtt{S}$ grows exponentially with time and number of agents.
Instead, \citet{Dibangoye:OMDP:2016} build on the piecewise linearity and convexity in the simultaneous-move occupancy-state space of the $\epsilon$-optimal value function.
Also, they introduced simultaneous-move backup operators that can circumvent the exhaustive enumeration of all joint decision rules using mixed-integer linear programs. 
They provided equivalence relations among private histories to enhance value generalization.
Altogether, these operations made it possible to use single-agent algorithms, \eg oSARSA \citep{DibangoyeBICML18} and point-based value iteration \citep{Dibangoye:OMDP:2016} algorithms, to solve $\mathtt{M}$ while preserving the convergence guarantees.
However, scalability remains a major issue. The following lemma suggests there is room for greater scalability by breaking down the decision-making process one step and one agent at a time. 

\begin{restatable}[]{lem}{lemsequenceapproach}[Proof in Appendix C.1]
\label{lem:sequence:approach}
Let $\sigma\colon \mathbb{N}^*_{\leq n} \mapsto \mathbb{N}^*_{\leq n}$ be a permutation over agents.  \citeauthor{bellman}'s optimality equations (\ref{eq:bellman}) can be re-written in sequential form with no loss in optimality, \ie   for all $t\in \mathbb{N}_{\leq \ell}$ and any $t$-th simultaneous-move occupancy state $\mathtt{s}_t$,
\begin{align*}
\upsilon_{\mathtt{M}',\gamma,t}^*(\mathtt{s}_t) = \max_{\mathtt{a}_t^{\sigma(1)}\in \mathtt{A}_t^{\sigma(1)}}  \cdots \max_{\mathtt{a}_t^{\sigma(n)}\in \mathtt{A}_t^{\sigma(n)}} ~\mathtt{q}_{\mathtt{M}',\gamma,t}^*(\mathtt{s}_t,\mathtt{a}_t).
\end{align*} 
%
%
\end{restatable}
The main issue with the dominant paradigm is the necessity for simultaneous decision-making rather than the double exponential growth of joint policies over time and the number of agents.
If simultaneous planning is conducted, then backups are performed collectively for all agents, all at once.
In the much larger problems that are our long-term objective, the double exponential growth of joint policies would render even a single simultaneous-move backup, \cf \citeauthor{bellman}'s optimality equations (\ref{eq:bellman}), prohibitively expensive.
Instead, Lemma \ref{lem:sequence:approach} suggests that planning can take place in an offline sequential central manner while preserving optimality, even though agents (online) execute actions in a decentralized and simultaneous fashion.
Doing so will disentangle the decision variables.
Even more importantly, it will make it possible to implement efficient basic operations, \eg occupancy-state estimation, value generalization, and backup, of exact single-agent algorithms. 
The following section introduces the sequential-move central planner approach to optimally solving Dec-POMDPs, assuming for the sake of clarity that the permutation over agents $\sigma$ is fixed throughout the planning process.
We shall demonstrate the sequential-move centralized training for decentralized execution paradigm opens the doors for more scalable single-agent algorithms to apply to multi-agent problems.

\section{Sequential-Move Central Planning}
\label{sec:scpe}

This section presents an alternative paradigm to (optimally) solving simultaneous-move Dec-POMDPs aimed at overcoming the limitations inherent to the dominant approach. We will establish that simultaneous-move Dec-POMDPs reduce to sequential-move Dec-POMDPs. This reduction enables us to apply techniques designed for sequential-move settings to their simultaneous counterparts, thereby introducing the sequential-move centralized training for decentralized control. For clarity, we will distinguish between multi-agent and single-agent formulations of sequential-move Dec-POMDPs, while also establishing their equivalence, which enables the transfer of theories and algorithms from single-agent to sequential-move multi-agent scenarios. 


\subsection{Sequential-Move Multi-Agent Formulation}

This subsection starts with a fairly general definition of $n$-agent sequential-move Dec-POMDPs. Theorem \ref{thm:simultaneous:sequential} provides the formal proof that one can convert any simultaneous-move Dec-POMDP into an equivalent sequential-move one, allowing us to solve the former using the latter.

\begin{definition}
\label{def:seq:decpomdp}
A $n$-agent sequential-move Dec-POMDP is given by a tuple $M\doteq (X, U, Z, \rho, p, r, \ell')$, where: $X$ is a finite set of states, denoted $x$ or $y$; $U(\tau)$ is a finite set of individual actions available to any arbitrary agent at decision epoch $\tau$, denoted $u$; $Z^i$ is a finite set of individual observations of agent $i$, denoted $z^i$, with $Z \doteq Z^1\times Z^2\times \cdots\times Z^n$ and $z \doteq (z^1,\cdots,z^n)$; $\rho\colon \mathbb{N}_{\leq \ell'-1} \to \mathbb{N}^*_{\leq n}$ is the agent function\footnote{It is worth noticing that agent function $\rho$ prescribes $(H^1,\ldots,H^n)$, with $H^i$ being the set of points in time agent $i$ takes actions, so that $\cap_{i\in \mathbb{N}^*_{\leq n}} H^i = \emptyset$, with $\cup_{i\in \mathbb{N}^*_{\leq n}} H^i =\mathbb{N}_{\leq\ell'-1}$.}, which assigns to each decision epoch $\tau\in \mathbb{N}_{\leq \ell'-1}$ an agent $i\in \mathbb{N}^*_{\leq n}$ who chooses an action at that step;  $p_{xy}^{uz}(\tau)$ is the $\tau$-th dynamics function that specifies the probability of the  state being $y$ and having observation $z$ after taking action $u$ in state $x$;  $r_{x,u}(\tau)$ is the $\tau$-th reward function that specifies the reward received after agent $\rho(\tau)$ takes action $u$ in state $x$; and $\ell'$ is the planning horizon.
\end{definition}

In sequential-move multi-agent formulation $M$, agents may start with noisy information regarding the state of the world, specified by the initial state distribution $b_0$. At each decision epoch $\tau$, the world is in some state $x_\tau$. Agent $\rho(\tau)$ chooses any action $u_\tau$ available in action space $U(\tau)$. Upon the execution of action $u_\tau$ in state $x_\tau$, the world immediately responds with the corresponding reward $r_{x_\tau,u_\tau}(\tau)$. It further responds at the subsequent decision epoch by randomly moving into a new state $x_{\tau+1}$ and provides each agent $i$ with sequential observation $z^i_{\tau+1}$, according to the dynamics function $p_{x_\tau,x_{\tau+1}}^{u_\tau,z_{\tau+1}}(\tau)$. The process repeats iteratively until the planning horizon is exhausted. Optimally solving $M$ aims at finding sequential-move joint policy $\pi$, \ie a sequence of private decision rules, $\pi \doteq (a_\tau)_{\tau\in \mathbb{N}_{\leq \ell'-1}}$, which maximizes the expected cumulative $\gamma$-discounted reward starting from initial state distribution $b_0$ onward, \ie 
\begin{align}
\upsilon_{M,\lambda,0}^\pi(b_0) &\doteq {\textstyle\mathbb{E}\{ \sum_{\tau=0}^{\ell'-1} \lambda(\tau)\cdot r_{x_\tau,u_\tau}(\tau)\mid b_0, \pi \}},
\label{eq:sequential:perf:criterion}
\end{align}
where $\lambda\colon \mathbb{N}_{\leq \ell'-1} \to (0,1)$ is the discount function. At every decision epoch $\tau$, agent $i = \rho(\tau)$ acts upon private decision rule $a_\tau\colon o^i_\tau\mapsto u_\tau$, which depends on sequential private histories. Sequential private histories are defined as $o^i_{\tau+1} = (o^i_\tau, u_\tau, z^i_{\tau+1})$ if agent $i= \rho(\tau)$; and $o^i_{\tau+1} = (o^i_\tau, z^i_{\tau+1})$ otherwise, with initial sequential private history being $o^i_0 \doteq \emptyset$. 
At each step $\tau\in \mathbb{N}_{\leq \ell'-1}$, we denote $O^i_\tau$ and $A_\tau$  the finite set of agent $i$'s $\tau$-th private histories $o^i_\tau\in O^i_\tau$  and decision rules $a_\tau\in A_\tau$, respectively.

Upon executing a sequential joint policy, agents collectively build a sequential joint history satisfying the following recursive formula: $o_{\tau+1} = (o_\tau, u_\tau, z_{\tau+1})$. %
At each step $\tau\in \mathbb{N}_{\leq \ell'-1}$, we denote $O_\tau$  the finite set of sequential joint histories $o_\tau\in O_\tau$.
Furthermore, each sequential joint history $o_\tau$ includes the sequential private history $o^i_\tau$ of each agent $i$. We denote the sequential private policy of agent $i$ extracted from sequential joint policy $\pi$ as $\pi^i \doteq (a_\tau)_{\tau\in H^i}$. Perhaps surprisingly, we establish that one can always convert any simultaneous-move Dec-POMDP into a sequential-move one with no loss in optimality. 

\begin{restatable}[]{thm}{thmsimultaneoussequential}[Proof in Appendix C.2]
\label{thm:simultaneous:sequential}
For every simultaneous-move Dec-POMDP $\mathtt{M}$, there exists a corresponding sequential-move Dec-POMDP $M$ along with the appropriate performance criterion (\ref{eq:sequential:perf:criterion}), whose optimal solution is also optimal for $\mathtt{M}$.
\end{restatable}
Theorem \ref{thm:simultaneous:sequential} establishes that any simultaneous-move Dec-POMDP can be solved as a sequential-move Dec-POMDP. One can similarly recast any simultaneous-move Dec-POMDP into an extensive-form game \citep{KOVARIK2022103645}. 
The main similarity is that both transformations allow the sequencing of decision-making one agent at a time. Yet, these transformations remain fundamentally different, mainly as sequential-move Dec-POMDPs make applying theory and algorithms derived from Markov decision processes easy.

\subsection{Sequential-Move Single-Agent Reformulation}
 
In the remainder, we assume a sequential central planner who knows model $M$ and selects what to do on behalf of all agents. Even though agents act and receive feedback about the state of the world simultaneously, as in $\mathtt{M}$, our sequential central planner can select the simultaneous private decision rule of one agent at a time by relying on the equivalent model $M$. To this end, it essentially traverses the space of joint decision rules by performing expansions of joint decision rules, one private decision rule, or one agent, at a time. This sequential expansion is illustrated in Figure 1, \cf Appendix. 
The selection of a private decision rule of an agent depends on the choice of private decision rules from previous points in time and the initial belief state about the process. We call sequential exhaustive data the overall data available to the sequential central planner at each sequential decision epoch.
 
\begin{definition}
\label{def:sequential:exhaustive:data}
For all sequential decision epoch $\tau\in \mathbb{N}_{\leq \ell'-1}$, sequential exhaustive data $\iota_\tau \doteq (b_0, a_{0:\tau-1})$.
\end{definition}

Sequential exhaustive data satisfy recursion $\iota_{\tau+1} = (\iota_\tau, a_\tau)$, with boundary condition $\iota_0 \doteq (b_0)$.
It is worth noticing that variables $(\iota_\tau)_{\tau\in \mathbb{N}_{\leq \ell'-1}}$ taking values in sequential exhaustive data set $\pmb{S}$ define a sequential-move data-driven Markov decision process $\pmb{M}$. In such a process, states are sequential exhaustive data and actions are private decision rules, as illustrated in the influence diagram in Figure 1, \cf Appendix. This process is deterministic since the next sequential exhaustive data is given by concatenating the previous one $\iota_\tau$ along with the sequential action choice $a_\tau$, \ie $\pmb{p}\colon (\iota_\tau, a_\tau) \mapsto \iota_{\tau+1}$. Furthermore, upon taking sequential action $a_\tau$ at sequential exhaustive data $\iota_\tau$, the expected reward is $\pmb{r}\colon (\iota_\tau, a_\tau) \mapsto \mathbb{E}\{\lambda(\tau) \cdot r_{x_\tau,u_\tau}(\tau)\mid \iota_\tau, a_\tau\}$.

\begin{definition}
\label{def:sequential:move:data:driven:mdp}
A sequential-move data-driven Markov decision process \wrt $M$ is given by $\pmb{M} \doteq \langle \pmb{S}, A,\pmb{p},\pmb{r} \rangle$ where: $\pmb{S}\doteq \cup_{\tau\in \mathbb{N}_{\leq \ell'-1}} \pmb{S}_\tau$ defines the state space, where $\pmb{S}_\tau$ is the set of all sequential exhaustive data at sequential decision epoch $\tau\in \mathbb{N}_{\leq \ell'-1}$; $A \doteq \cup_{\tau\in \mathbb{N}_{\leq \ell'-1}} A_\tau$ denotes the action space, where $A_\tau$ is the action space available at sequential decision epoch $\tau\in \mathbb{N}_{\leq \ell'-1}$; $\pmb{p}\colon \pmb{S}\times A \to \pmb{S}$ prescribes the next state $\iota_{\tau+1} = \pmb{p} (\iota_\tau, a_\tau)$ after taking action $a_\tau$ at state $\iota_\tau$;  $\pmb{r}\colon \pmb{S}\times A \to \mathbb{R}$ specifies the immediate reward $\pmb{r} (\iota_\tau, a_\tau)$ upon taking action $a_\tau$ in state $\iota_\tau$; and $\iota_0 \doteq (b_0)$. 
\end{definition}

Similarly to $M$, optimally solving $\pmb{M}$ aims at finding sequential joint policy $\pi$, which maximizes the cumulative $\gamma$-discounted reward starting from $\iota_0$ onward, and given by $\upsilon^\pi_{\pmb{M},\gamma, 0}(\iota_0) \doteq \mathbb{E}\{ \sum_{\tau=0}^{\ell'-1} \pmb{r} (\iota_\tau, a_\tau)\mid \iota_0, \pi \}$. Building upon theory and algorithms for MDPs, one can extract an optimal sequential joint policy $\pi^*$ from the solution of \citeauthor{bellman}'s optimality equations, \ie
\begin{align*}
    \upsilon^*_{\pmb{M}, \lambda, \tau}(\iota_\tau) &={\textstyle \max_{a_\tau\in A_\tau} q^*_{\pmb{M}, \lambda, \tau}(\iota_\tau, a_\tau)}\\
    q^*_{\pmb{M}, \lambda, \tau}(\iota_\tau, a_\tau) &= \pmb{r} (\iota_\tau, a_\tau) + \upsilon^*_{\pmb{M}, \lambda, \tau+1}( \pmb{p} (\iota_\tau, a_\tau))
\end{align*}
with boundary condition $\upsilon^*_{\pmb{M}, \lambda, \ell'}(\cdot) = q^*_{\pmb{M}, \lambda, \ell'}(\cdot,\cdot) \doteq 0$. To transfer knowledge from sequential-move data-driven MDP $\pmb{M}$ to corresponding sequential-move Dec-POMDP $M$, it will prove useful to show $\pmb{M}$ and $M$ are equivalent.

\begin{restatable}[]{lem}{lemequivalencesequential}[Proof in Appendix C.3]
\label{lem:equivalence:sequential}
If we let $\pmb{M}$ be a sequential-move data-driven MDP \wrt original problem $M$, then any optimal sequential joint policy $\pi^*_{\pmb{M}}$ for $\pmb{M}$ is also an optimal solution for  $M$. 
\end{restatable}

$\pmb{M}$ differs from $M$ in that the search space of $\pmb{M}$ is made explicit, \ie a state in $\pmb{M}$ (or a sequential exhaustive data) is a point in the search space. Consequently, space $\pmb{S}$ is too large to be generated and stored in memory. It is not useful to remember all the states the sequential central planner experienced. Instead, one can rely on more concise representations, often referred to as statistics. Analogously to \citet{Dibangoye:OMDP:2016}, we introduce a statistic that summarizes the sequential exhaustive data, and name that statistic sequential-move occupancy state (SOC). 

\begin{definition}
\label{def:soc}
A SOC $s_\tau\in \triangle(X \times O_\tau)$ at step $\tau\in \mathbb{N}_{\leq \ell'-1}$ is defined as a posterior probability distribution of hidden states and sequential joint histories given sequential exhaustive data $\iota_\tau$, \ie 
$s_\tau\colon (x_\tau,o_\tau) \mapsto \Pr\{x_\tau,o_\tau| \iota_\tau\}$.
\end{definition}

Previously, \citet{Dibangoye:OMDP:2016} introduced the concept of (simultaneous-move) occupancy state to select a simultaneous joint decision rule on behalf of all agents, all at once. Instead, SOCs make it possible to select a private decision rule for each agent, one agent at a time. However, not all statistics can replace the sequential exhaustive data while preserving the ability to find an optimal solution. Statistics that exhibit such property include sufficient statistics of sequential exhaustive data for optimal decision-making in sequential-move Dec-POMDP $M$. Before further proving the sufficiency of SOCs, we start with two simple yet important preliminary lemmas.

\begin{restatable}[]{lem}{lemmarkov}[Proof in Appendix C.4]
\label{lem:markov}
SOCs $(s_\tau)_{\tau\in \mathbb{N}_{\leq \ell'-1}}$ describe a Markov process. In other words, there exists $T\colon (s_\tau, a_\tau) \mapsto s_{\tau+1}$, \ie for every state-history pair $\theta =(x_{\tau+1}, \langle o_\tau, u_\tau, z_{\tau+1}\rangle)$,
\begin{align*}
s_{\tau+1}(\theta) &= \textstyle
 \sum_{x_\tau\in X} p_{x_\tau,x_{\tau+1}}^{a_\tau( o^{\rho(\tau)}_\tau), z_{\tau+1}}(\tau)  \cdot 
 s_\tau(x_\tau,o_\tau).
\end{align*}
\end{restatable}
Lemma \ref{lem:markov} describes the transition rule of a continuous-state deterministic Markov decision process $M'$, in which states and actions are SOCs and private decision rules, respectively. Next, we define the expected immediate reward model of $M'$ that SOCs and private decision rules induce.

\begin{restatable}[]{lem}{lemreward}[Proof in Appendix C.5]
\label{lem:reward}
SOCs $(s_\tau)_{\tau\in \mathbb{N}_{\leq \ell'-1}}$ are sufficient statistics of sequential exhaustive data $(\iota_\tau)_{\tau\in \mathbb{N}_{\leq \ell'-1}}$ for estimating expected immediate reward. In other words, there exists a mapping $R\colon (s_\tau, a_\tau) \mapsto \mathbb{E}\{\lambda(\tau)\cdot r_{x_\tau, u_\tau}(\tau)\mid s_\tau, a_\tau\}$ such that for every sequential exhaustive data $\iota_\tau$, corresponding SOC $s_\tau$, and any private decision rule $a_\tau$, we have:
\begin{align*}
R(s_\tau, a_\tau) \doteq \pmb{r}(\iota_\tau, a_\tau).   
\end{align*}
\end{restatable}
We are now ready to define sequential-move occupancy Markov decision process ($so$MDP) $M'$ the Markov decision process that SOCs describe via dynamics and reward models $T$ (\cf Lemma \ref{lem:markov}) and $R$ (\cf Lemma \ref{lem:reward}), respectively.
\begin{definition}
A $so$MDP \wrt $M$ is given by a tuple $M'\doteq ( S, A, T, R,\ell' )$ where: $S \subset \triangle(X\times O)$ is the space of SOCs; $A$ is as in $M$; $T\colon S\times A\to S$ is the transition rule; and $R\colon S\times A\to \mathbb{R}$ is the immediate reward model.
\end{definition}

The goal of $so$MDP $M'$ is to find an optimal sequential joint policy $\pi$ which maximizes the cumulative $\gamma$-discounted reward starting from $s_0$ onward, \ie
$
\upsilon^\pi_{M',\lambda, 0}(s_0) \doteq {\textstyle \mathbb{E}\{\sum_{\tau=0}^{\ell'-1}  R(s_\tau, a_\tau) \mid s_0,\pi\}}.
$
The optimal state- and action-value functions $(\upsilon^*_{M',\lambda, \tau}, q^*_{M',\lambda, \tau})_{\tau\in \mathbb{N}_{\leq \ell'-1}}$ of $M'$ are solutions of \citeauthor{bellman}'s optimality equations: for every $s_\tau\in S$, 
\begin{align*}
\upsilon^*_{M', \lambda, \tau}(s_\tau) &={\textstyle \max_{a_\tau\in A_\tau} q^*_{M', \lambda, \tau}(s_\tau, a_\tau)},\\
q^*_{M', \lambda, \tau}(s_\tau, a_\tau) &= R(s_\tau, a_\tau) + \upsilon^*_{M', \lambda, \tau+1}( T (s_\tau, a_\tau)),
\end{align*}
with boundary condition $\upsilon^*_{M', \lambda, \ell'}(\cdot) = q^*_{M, \lambda, \ell'}(\cdot,\cdot) \doteq 0$. Given the optimal value functions, one can extract an optimal sequential joint policy $\pi^*$ as follows: for every decision epoch $\tau\in \mathbb{N}_{\leq \ell}$, and SOC $s^*_\tau \doteq (\Pr\{x_\tau, o_\tau | s_0,\pi^*\})_{x_\tau\in X, o_\tau\in O}$,  we have $a_\tau^* = \argmax_{a_\tau\in A_\tau} q^*_{M', \lambda, \tau}(s^*_\tau, a_\tau)$. Interestingly, an optimal sequential joint policy for $so$MDP $M'$ is also an optimal joint policy for the original sequential-move Dec-POMDP $M$.

\begin{restatable}[]{thm}{thmsufficiencycondition}[Proof in Appendix C.6]
\label{thm:sufficiency:condition}
Let $M$ be a sequential-move Dec-POMDP. Let $\pmb{M}$ be a sequential-move data-driven MDP \wrt $M$. SOCs $(s_\tau)_{\tau\in \mathbb{N}_{\leq \ell'-1}}$ are statistics sufficient to replace the sequential exhaustive data $(\iota_\tau)_{\tau\in \mathbb{N}_{\leq \ell'-1}}$ via $so$MDP $M'$ \wrt $M$ while still preserving the ability to optimally solve $\pmb{M}$ (resp. $M$).
\end{restatable}

Theorem \ref{thm:sufficiency:condition} demonstrates that---by optimally solving any problem $M'$, $\pmb{M}$ or $M$ along with the appropriate performance criteria---we are guaranteed to find an optimal solution for the others. 

\section{Exploiting Structural Results} \label{sec:structural_properties}

The previous section shows that---to solve simultaneous-move Dec-POMDP $M$---one can equivalently solve the corresponding $so$MDP $M'$. This section shows optimal sequential state- and action-value functions $(\upsilon^*_{M',\lambda, \tau}, q^*_{M',\lambda, \tau})_{\tau\in \mathbb{N}_{\leq \ell'-1}}$ for $M'$ are piecewise linear and convex functions of states and actions in $M'$. These convexity properties enable us to generalize values experienced in a few states and actions over the entire state and action spaces. Doing so speeds up convergence towards an optimal solution for $M'$ (respectively, $M$). Similar properties exist for simultaneous-move Dec-POMDPs as well. However, maintaining sequential action-value functions will prove more tractable---\ie point-based simultaneous-move backups are double exponential, whereas sequential-move ones exhibit only polynomial-time complexity in the worst case. 

Before we state one of the main results of this paper---the piecewise linearity and convexity of optimal sequential state- and action-value functions over state and action spaces---we introduce the following preliminary definition of sequential state-action occupancy measures and related properties.
Let $\hat{s}_\tau$ and $\hat{a}_\tau$ be extended variants of SOCs $s_\tau$ and private decision rules $a_\tau$, respectively. In other words, for SOC $s_\tau$ and private decision rule $a_\tau$, we have $\hat{s}_\tau\colon (x_\tau,o_\tau, u_\tau) \mapsto s_\tau(x_\tau,o_\tau)$ and $\hat{a}_\tau\colon (x_\tau,o_\tau, u_\tau) \mapsto \Pr\{u_\tau | a_\tau, o^{\rho(\tau)}_\tau\}$. We denote $\hat{S}$ and $\hat{A}$ the extended spaces of SOCs and private decision rules, respectively. We further let $\hat{S}\odot \hat{A}$ be the space of all Hadamard products $\hat{s}_\tau\odot \hat{a}_\tau$ between extended SOC $\hat{s}_\tau$ and private decision rule $\hat{a}_\tau$, that is sequential state-action occupancy measure, where $[\hat{s}_\tau\odot \hat{a}_\tau](x_\tau,o_\tau, u_\tau) = \hat{s}_\tau (x_\tau,o_\tau, u_\tau)\cdot \hat{a}_\tau (x_\tau,o_\tau, u_\tau)$.

\begin{restatable}[]{thm}{thmpwlc}[Proof in Appendix D.5]
\label{thm:pwlc}
For any $so$MDP $M'$, optimal state-value functions $\upsilon^*_{M',\lambda, 0:\ell'}$ are piecewise linear and convex \wrt SOCs. Furthermore, optimal action-value functions $q^*_{M',\lambda, 0:\ell'}$ are also piecewise linear and convex \wrt sequential state-action occupancy measures.
\end{restatable}

Theorem \ref{thm:pwlc} proved the optimal state-value function for $M'$ is piecewise linear and convex. We now introduce a practical representation of piecewise linear and convex optimal state- and action-value functions $\upsilon^*_{M',\lambda, 0:\ell'}$ and $q^*_{M',\lambda, 0:\ell'}$, respectively. We show they can be represented as the upper envelope of a finite collection of linear functions. Such a representation, suitable only for lower-bound state- and action-value functions, is referred to as the max-plane representation \citep{smith07:thesis}. In addition, we describe operations for initializing, updating and selecting greedy decisions. 

\begin{restatable}[]{cor}{corlowerbound}[Proof in Appendix F.1]
\label{cor:lower:bound}
For any $so$MDP $M'$, let $\upsilon^*_{M',\lambda, 0:\ell'}$ be the optimal state-value functions \wrt SOCs. Then, for every $\tau\in \mathbb{N}_{\leq \ell'-1}$, there exists a finite collection of linear functions  \wrt SOCs, $V_\tau \doteq \{\alpha_\tau^{(\kappa)} \colon \kappa \in \mathbb{N}_{\leq k} \}$  such that:
\begin{align*}
    \upsilon^*_{M',\lambda, \tau}\colon s_\tau &\mapsto\textstyle  \max_{\alpha_\tau^{(\kappa)}\in V_\tau} \langle s_\tau, \alpha_\tau^{(\kappa)}\rangle
\end{align*}
where $\langle \cdot, \cdot\rangle$ denotes the inner product.
\end{restatable}

Next,  it will prove useful to define a constant private decision rule $a^{u_\tau}_\tau(\cdot) = u_\tau$, which prescribes $u_\tau$ for any private history, and the action-value linear function $\beta_\tau^{(\kappa)}$ that state-value linear function $\alpha_{\tau+1}^{(\kappa)}$ induces, \ie for any $(x_\tau, o_\tau, u_\tau) $,
\begin{align*} 
\beta_\tau^{(\kappa)}  (x_\tau, o_\tau, u_\tau) = \mathbb{E}\{ \lambda(\tau) r_{x_\tau,u_\tau}(\tau) +  \alpha_{\tau+1}^{(\kappa)}(x_{\tau+1}, o_{\tau+1})\}.
\end{align*} 
We are now ready to describe how to greedily select a private decision rule from a given state-value function.

\begin{restatable}[]{thm}{thmlowerbound}[Proof in Appendix F.2]
\label{thm:lower:bound}
For any $so$MDP $M'$, let $V_{\tau+1}$ be a finite collection of linear functions \wrt SOCs providing the max-plane representation of state-value function $\upsilon_{M',\lambda, \tau+1}$. Then, it follows that greedy private decision rule $a_\tau^*$ at SOC $s_\tau$ is:
\begin{align*}
&(a_\tau^{s_\tau}, \beta_\tau^{s_\tau}) \in \textstyle { \argmax_{(a_\tau^{(\kappa)}, \beta_\tau^{(\kappa)})\colon \kappa \in \mathbb{N}_{\leq k}} \langle \hat{s}_\tau \odot \hat{a}_\tau^{(\kappa)}, \beta_\tau^{(\kappa)}\rangle},\\
&a_\tau^{(\kappa)}(o^{\rho(\tau)}_\tau) \in \textstyle \argmax_{u_\tau\in U(\tau)} \langle \hat{s}_\tau \odot \hat{a}^{u_\tau}_\tau,\beta_\tau^{(\kappa)} \rangle.
\end{align*} 
Also, $\alpha_\tau^{s_\tau}\colon (x_\tau, o_\tau) \mapsto \beta_\tau^{s_\tau}(x_\tau, o_\tau, a_\tau^{s_\tau}(o^{\rho(\tau)}_\tau))$ is the greedy linear function induced by $ \beta_\tau^{s_\tau}$. The update of finite collection $V_\tau$ is as follows: $V_\tau \gets V_\tau \cup \{\alpha_\tau^{s_\tau}\}$.
\end{restatable}

Theorem \ref{thm:lower:bound} greedily selects a decision rule $a_\tau^{s_\tau}$ according to a lower-bound value function at a given SOC $s_\tau$ and a linear action-value function $\beta_\tau^{s_\tau}$ to update the lower-bound value function $\upsilon_{M',\lambda, \tau}$. It then computes a linear state-value function $\alpha_\tau^{s_\tau}$ over SOCs to be added in the collection of linear functions $V_\tau$. These update rules exhibit a complexity that is polynomial only \wrt the sizes of SOC $s_\tau$ and set $V_\tau$, \cf Lemma 5 in the appendix. Similar rules have been previously introduced for the simultaneous-move setting. Unfortunately, these rules relied on mixed-integer linear programs in the best case. In the worst case, they enumerate all joint decision rules, thus creating a double exponential time complexity, \cf Lemma 5 in the appendix.

\section{The $o$SARSA Algorithm} \label{sec:oSARSA}

This section extends a reinforcement learning algorithm for $\epsilon$-optimally solving simultaneous-move Dec-POMDPs, known as $o$SARSA  \citep{DibangoyeBICML18}, to sequential centralized training for decentralized execution. 

$o$SARSA generalizes SARSA from MDPs to $o$MDPs. It iteratively improves a (sequential) policy and its corresponding linear action-value functions, one per time step. The improved (sequential) policies are constructed by generating trajectories of OCs, one per episode, guided by an $\epsilon$-greedy exploration strategy. A mixed-integer linear program (MILP) is used to select greedy joint decision rules to avoid the enumeration of double-exponential joint decision rules at every step. Once a trajectory resumes, it updates the linear action-value functions in the reversed order in which OCs were visited in it. We chose $o$SARSA among all point-based algorithms because it leverages the piecewise linearity and convexity properties of optimal value functions in Dec-POMDPs, yet maintains only a single linear value function per time step. Besides, it is guaranteed to find $\epsilon$-optimal policy asymptotically. Note that algorithms that do not restrict to the max-plane representations for value functions, \eg the feature-based value iteration (FB-HSVI) algorithm \citep{Dibangoye:OMDP:2016}, may not fully exploit our findings.

\begin{algorithm}[t]
    \caption{$o$SARSA$(\epsilon)$.}
    \label{algo:oSARSA}

        \SetKwFor{ForEachInit}{foreach}{initialize}{}
        
        \BlankLine
        Init $a_{0:\ell'-1}\gets$ blind policy, $\alpha_{\cdot}(\cdot) \gets 0$, $g_{\cdot} \gets -\infty$. \\
        
    \ForEach{episode}{
            Init SOC $s_0 \gets b_0$ and step $\tau_0\gets 0$.\\ 
            \For{$\tau = 0$ \textbf{to} $\ell'-1$}{
                Select\label{algo:line:eps-greedy} $\epsilon$-greedily $a_\tau^{s_\tau}$ \wrt $\alpha_{\tau+1}$.\\
                Compute SOC $s_{\tau+1} \gets     T(s_\tau, a^{s_\tau}_\tau)$. \\
                \label{algo:line:accept}\If{ $\pmb{\mathtt{Accept}}(a_\tau^{s_\tau},  g_{\tau+1}, \alpha_{\tau+1}(s_{\tau+1}))$   }{
                    $a_{0:\ell'-1} \gets \langle a_{0:\tau-1},a_\tau^{s_\tau}, a_{\tau+1:\ell'-1} \rangle$.  \\
                    $g_{\tau+1} \gets \alpha_{\tau+1}(s_{\tau+1})$.\\
                    $\tau_0 \gets \tau$. \\
                }
            }
            \For{$\tau = \tau_0$ \textbf{to} $0$}{
                Update $\alpha_\tau$ backward.\\
            }
        }
        
\end{algorithm}

To leverage the sequential centralized training for the decentralized execution paradigm, we apply the $o$SARSA algorithm in $M'$ corresponding to $\mathtt{M}$. Hence, Algorithm \ref{algo:oSARSA} proceeds by generating trajectories of SOCs instead of OCs, one trajectory per episode, guided by sequential $\epsilon$-greedy exploration steps. The sequential greedy selection rule requires only a polynomial time complexity, \cf Theorem \ref{thm:lower:bound}. We use a portfolio of heuristic policies to guide the selection of the next sequential action to enhance exploration. This portfolio includes the random policy, the underlying MDP policy, and the blind policy \citep{Hauskrecht00}. We introduce an acceptance rule based on Simulated Annealing (SA) \citep{rutenbar1989simulated} to avoid the algorithm being stuck in a local optimum. Therein, a modification of the current sequential policy is kept not only if the performance improves but also in the opposite case, with a probability depending on the loss and on a temperature coefficient decreasing with time. While the point-based greedy selection rules in the sequential centralized training for decentralized execution paradigm are faster than those in its simultaneous counterpart, the length of trajectories increases by order of $n$, \ie the number of agents. In most domains, however, the complexity of a single point-based greedy selection affects $o$SARSA far more strongly than the length of trajectories. The complexity of point-based greedy selection is a far better predictor of the ability to solve a simultaneous-move Dec-POMDP using $o$SARSA than the length of trajectories.

\section{Experiments}
\label{sec:oSarsa_results}

This section presents the results of our experiments, which were carried out to juxtapose the sequential planning approach with its simultaneous counterpart employed in many leading-edge global multi-agent planning and reinforcement learning algorithms, encompassing the utilization of $o$SARSA, \cf \Cref{algo:oSARSA}, as a standard algorithmic scheme. Our empirical analysis involves two variants of the $o$SARSA algorithm, namely $o$SARSA$^\mathtt{sim}$ and $o$SARSA$^\mathtt{seq}$, each employing a distinct reformulation of the original simultaneous-move Dec-POMDP $\mathtt{M}$ and point-based backup method. $o$SARSA$^\mathtt{sim}$ relies on $o$MDP $\mathtt{M}'$ \wrt $\mathtt{M}$ and utilizes mixed-integer linear programs (MILPs) for implicit enumeration of joint decision rules. We used ILOG CPLEX Optimization Studio to solve the MILPs. $o$SARSA$^\mathtt{seq}$, instead, relies on $so$MDP $M'$ \wrt $\mathtt{M}$ and utilizes point-based sequential backup operator introduced in Theorem \ref{thm:lower:bound}. We used the same heuristics (random, blind, and MDP policies) for $o$SARSA$^\mathtt{sim}$ and $o$SARSA$^\mathtt{seq}$. For reference, we also reported, when available, the performance of state-of-the-art $\epsilon$-optimal solver for two-agent simultaneous-move Dec-POMDP $\mathtt{M}$, \ie FB-HSVI. Unfortunately, most global methods are not geared to scale up with the number of agents. To present a comprehensive view, we have also compared our results against local policy- and value-based reinforcement learning methods, \ie advantage actor-critic (A2C) \citep{Konda:NIPS00} and independent $Q$-learning (IQL) \citep{Tan:1997:MRL:284860.284934}.  All experiments were run on an Ubuntu machine with 16 GB of available RAM and an 1.8 GHz processor, using only one core and a time limit of $1$ hour. The source code is available at \url{https://git.lwp.rug.nl/ml-rug/osarsa-aaai-25}.

We have comprehensively assessed various algorithms using several two-agent benchmarks from academic literature, available at \url{masplan.org}. These benchmarks include mabc, recycling, grid3x3, boxpushing, mars, and tiger. To enable the comparison of multiple agents, we have also used the multi-agent variants of these benchmarks, \cf \citep{peralez2024solving}.
Please refer to the appendix for a detailed description of these benchmarks. 
Our empirical study aimed to assess the superiority of the sequential planning approach through the drop in the complexity of the point-based backup operators. Each experience in reinforcement learning algorithms, \ie A2C, IQL and $o$SARSA, was repeated with three different seeds. The values reported in \Cref{table:results_2players,table:results_multi}, are the best solution obtained among the three trials for two- and many-agent domains, respectively.


\begin{table}[ht!]
\scriptsize

\begin{tabular}{c c c c c c}

\toprule%
 \centering%
 \bfseries $\ell$
 & \bfseries $o$SARSA$^\mathtt{seq}$
 & \bfseries $o$SARSA$^\mathtt{sim}$
 & \bfseries A2C
 & \bfseries IQL
 & \bfseries HSVI
 \\

\midrule
\multicolumn{6}{c}{tiger} \\
\midrule
$10$ & $\highest{15.18}$ & $\highest{15.18}$ &$-0.78$ &$-1.30$ &$\highest{15.18}$\\ 
\myrowcolour
$20$ &$\highest{30.37}$ &$\highest{30.37}$ &$-25.08$ &$-14.61$ &$28.75$\\ 
$40$ &$\highest{67.09}$ &$\highest{67.09}$ &$-62.23$ &$-50.72$ &$\highest{67.09}$\\ 
\myrowcolour
$100$ &$\highest{170.91}$ &$169.30$ &$-184.71$ &$-165.42$ &$\highest{170.90}$\\
\midrule
\multicolumn{6}{c}{recycling} \\
\midrule
$10$ &$\highest{31.86}$ &$\highest{31.86}$ &$31.48$ &$31.48$ &$\highest{31.86}$\\ 
\myrowcolour
$20$ &$\highest{62.63}$ &$\highest{62.63}$ &$62.25$ &$\highest{62.63}$ &n.a.\\ 
$40$ &$\highest{124.17}$ &$\highest{124.17}$ &$123.79$ &$123.79$ &n.a.\\ 
\myrowcolour
$100$ &$\highest{308.79}$ &$\highest{308.79}$ &$308.40$ &$308.72$ &$\highest{308.78}$\\ 
\midrule
\multicolumn{6}{c}{gridsmall} \\
\midrule
$10$ &$\highest{6.03}$ &$\highest{6.03}$ &$4.99$ &$5.31$ &$\highest{6.03}$\\ 
\myrowcolour
$20$ &$\highest{13.96}$ &$13.90$ &$11.16$ &$11.49$ &$13.93$\\ 
$40$ &$\highest{30.93}$ &$29.89$ &$22.56$ &$23.54$ &$28.55$\\ 
\myrowcolour
$100$ &$\highest{78.37}$ &$78.31$ &$56.92$ &$47.79$ &$75.92$\\ 
\midrule
\multicolumn{6}{c}{grid3x3} \\
\midrule
$10$ &$\highest{4.68}$ &$\highest{4.68}$ &$\highest{4.68}$ &$\highest{4.68}$ &$\highest{4.68}$\\ 
\myrowcolour
$20$ &$\highest{14.37}$ &$\highest{14.37}$ &$13.37$ &$\highest{14.36}$ &$\highest{14.35}$\\ 
$40$ &$\highest{34.35}$ &$\highest{34.35}$ &$\highest{32.34}$ &$\highest{34.35}$ &$\highest{34.33}$\\ 
\myrowcolour
$100$ &$\highest{94.35}$ & \sc oot & $92.34$ &$94.32$ &$94.24$\\ 
\midrule
\multicolumn{6}{c}{boxpushing} \\
\midrule
$10$ &$\highest{224.26}$ &$219.19$ &$54.69$ &$223.48$ &$223.74$\\ 
\myrowcolour
$20$ &$\highest{470.43}$ &$441.98$ &$123.59$ &$254.41$ &$458.10$\\ 
$40$ &$\highest{941.07}$ &$918.62$ &$236.79$ &$283.79$ &$636.28$\\ 
\myrowcolour
$100$ &$\highest{2366.21}$ &$1895.16$ &$599.97$ &$560.16$ &n.a.\\ 
\midrule
\multicolumn{6}{c}{mars} \\
\midrule
$10$ &$\highest{26.31}$ &$24.47$ &$17.90$ &$17.63$ &$\highest{26.31}$\\ 
\myrowcolour
$20$ &$\highest{52.32}$ &$52.20$ &$34.43$ &$35.27$ &$52.13$\\ 
$40$ &$\highest{104.07}$ &$103.25$ &$67.74$ &$65.94$ &$103.52$\\ 
\myrowcolour
$100$ &$\highest{255.18}$ & \sc oot & $152.52$ &$124.73$ &$249.92$\\ 
\midrule
\multicolumn{6}{c}{mabc} \\
\midrule
$10$ &$\highest{9.29}$ &$\highest{9.29}$ & $9.20$ & $9.20$ & $\highest{9.29}$\\ 
\myrowcolour
$20$ &$\highest{18.31}$ & $\highest{18.31}$ & $18.10$ & $18.20$ & n.a.\\ 
$40$ &$\highest{36.46}$ & $\highest{36.46}$ & $36.10$ & $36.20$ & n.a.\\ 
\myrowcolour
$100$ &$\highest{90.76}$ & $\highest{90.76}$ & $90.20$ & $90.20$  &$\highest{90.76}$\\ 

\bottomrule
\end{tabular}
\caption{For each two-agent domain and planning horizon $\ell$, we report the best value per algorithm. {\sc oot} means time limit of 1 hour has been exceeded and 'n.a.' is not available.}
\label{table:results_2players}
\end{table}

\Cref{table:results_2players} shows that $o$SARSA$^\mathtt{seq}$ outperforms  $o$SARSA$^\mathtt{sim}$ on all tested two-agent domains, sometimes by a significant margin. On boxpushing at planning horizon $\ell=100$, for instance, $o$SARSA$^\mathtt{seq}$ achieves value $2366.21$ against $1895.16$ for $o$SARSA$^\mathtt{sim}$ . It improves the best known values provided by FB-HVI in all tested domains, except for recycling and mabc, where it achieves equivalent values. Moreover, the margin between $o$SARSA$^\mathtt{seq}$ and FB-HSVI increases as the planning horizon increases. In boxpushing, the margin between $o$SARSA$^\mathtt{seq}$ and FB-HSVI increases from $0.52$ at planning horizon $\ell=10$ to $304.79$ at planning horizon $\ell=40$. It is worth noticing that $o$SARSA$^\mathtt{sim}$ achieves competitive performances against FB-HSVI in most tested two-agent domains, except mars and grid3x3, where it runs out of time. A2C and IQL are competitive on two weakly-coupled domains, \eg recycling and grid3x3, but get stuck in local optima in the other domains. \Cref{table:results_multi}  shows that $o$SARSA$^\mathtt{seq}$ outperforms $o$SARSA$^\mathtt{sim}$, A2C and IQL on all tested many-agent domains over different planning horizons. For medium to large teams and planning horizons,  $o$SARSA$^\mathtt{sim}$ runs out of time. Even though A2C and IQL scale up with larger teams, they achieve poor performances against $o$SARSA$^\mathtt{seq}$.

\begin{table}[t]
\scriptsize

\begin{tabular}{c c c c c c c}
\toprule%
 \centering%
 \bfseries $n$
 & \bfseries $\ell$
 & \bfseries $o$SARSA$^\mathtt{seq}$
 & \bfseries $o$SARSA$^\mathtt{sim}$
 & \bfseries A2C
 & \bfseries IQL
 \\


\midrule
\multicolumn{6}{c}{tiger} \\
\midrule
$3$ & $10$ &$\highest{11.29}$ & \sc oot & $-19.26$ &$-9.82$\\ 
\myrowcolour
$4$ & $10$ &$\highest{6.80}$ & \sc oot & $-110.08$ &$-18.48$ \\ 
$5$ & $2$ &$\highest{-4.00}$ &$\highest{-4.00}$ &$-30.00$ &$\highest{-4.00}$\\ 
\myrowcolour
$5$ & $4$ &$\highest{3.84}$ & \sc oot & $-50.00$ &$-7.01$ \\ 
$5$ & $6$ &$\highest{-0.16}$ & \sc oot & $-93.83$ &$-11.56$ \\ 
\myrowcolour
$5$ & $8$ &$\highest{-5.43}$ & \sc oot & $-128.98$ &$-14.97$ \\ 
$5$ & $10$ &$\highest{2.41}$ & \sc oot & $-126.62$ &$-131.91$ \\ 
\midrule
\multicolumn{6}{c}{recycling} \\
\midrule
$3$ & $10$ &$\highest{85.23}$ &$\highest{85.23}$ &$45.83$ &$47.51$\\ 
\myrowcolour
$4$ & $10$ &$\highest{108.92}$ &$105.96$ &$57.70$ &$55.49$ \\ 
$5$ & $10$ &$\highest{133.84}$ &$\highest{133.84}$ &$74.50$ &$72.88$ \\ 
\myrowcolour
$6$ & $10$ &$\highest{159.00}$ & \sc oot & $86.37$ &$79.83$ \\ 
$7$ & $2$ &$\highest{45.50}$ &\sc oot & $18.20$ &$20.80$\\ 
\myrowcolour
$7$ & $4$ &$\highest{80.50}$ & \sc oot & $57.11$ &$40.84$\\ 
$7$ & $6$ &$\highest{115.50}$ & \sc oot & $64.49$ &$59.06$\\ 
\myrowcolour
$7$ & $8$ &$\highest{150.50}$ & \sc oot & $72.97$ &$71.84$\\ 
$7$ & $10$ &$\highest{185.50}$ & \sc oot & $85.07$ &$91.07$\\ 
\midrule
\multicolumn{6}{c}{gridsmall} \\
\midrule
$3$ & $10$ &$\highest{5.62}$ & \sc oot & $0.34$ &$0.57$ \\ 
\myrowcolour
$4$ & $2$ &$\highest{0.13}$ &$\highest{0.13}$ &$0.02$ &$\highest{0.13}$\\ 
$4$ & $4$ &$\highest{0.78}$ & \sc oot & $0.22$ &$0.55$ \\ 
\myrowcolour
$4$ & $6$ &$\highest{1.75}$ & \sc oot & $0.39$ &$0.77$ \\ 
$4$ & $8$ &$\highest{2.85}$ & \sc oot & $0.65$ &$1.11$ \\ 
\myrowcolour
$4$ & $10$ &$\highest{4.09}$ & \sc oot & $1.50$ &$1.84$ \\ 

\bottomrule
\end{tabular}
\caption{For each many-agent domain and planning horizon $\ell$, we report the best value per algorithm. {\sc oot} means time limit of 1 hour has been exceeded.}
\label{table:results_multi}
\end{table}

The empirical results of $o$SARSA$^\mathtt{seq}$ are due to two interdependent reasons: (1) the piecewise linearity and convexity property of the value function over (sequential occupancy-) states and (2) the polynomial-time update rule.
The uniform continuity property (1) generalizes values from seen states to unseen ones. Doing so guides algorithms toward the most promising parts of the search space while discarding others. Although uniform continuity holds for both $o$SARSA$^\mathtt{sim}$ and $o$SARSA$^\mathtt{seq}$ (albeit in different state spaces), its efficiency differs. In particular, $o$SARSA$^\mathtt{seq}$ can discard parts of the search space at each decision step, but $o$SARSA$^\mathtt{sim}$ can only do it after $n$ decision steps. In other words, even if the search space is the same, $o$SARSA$^\mathtt{seq}$ will often explore fewer states than $o$SARSA$^\mathtt{sim}$.
The exponential complexity drop enables $o$SARSA$^\mathtt{seq}$ to perform exponentially faster episodes than $o$SARSA$^\mathtt{sim}$, thus visiting many states that will, in return, help discard larger parts of the search space. However, each update in a single state in $o$SARSA$^\mathtt{sim}$ is double exponential in the worst case. So, as the number of agents and actions per agent increases, performing even a single update (not to mention an episode) in $o$SARSA$^\mathtt{sim}$ is prohibitive. This insight hinders its ability to exploit the uniform continuity property in larger instances.

\section{Discussion}
\label{sec:conclusion}

This paper highlights the need for a paradigm shift in $\epsilon$-optimally solving Dec-POMDPs, transitioning from simultaneous- to sequential-move centralized training for decentralized execution. This shift addresses the silent coordination dilemma and the decision entanglement problems, which have hindered the scalability of $\epsilon$-optimal methods. Specifically, it leads to an exponential drop in the complexity of backups, although it comes at the cost of longer planning horizons. It facilitates the application of efficient single-agent methods while preserving their theoretical guarantees, as evidenced by the superior performance of the $o$SARSA$^\mathtt{seq}$ algorithm compared to all other tested baselines across domains involving two or more agents. Interestingly, this novel paradigm also addresses the multi-agent credit assignment problem, breaking down the optimal value functions among agents. This insight is important because this problem has recently attracted much attention in the multi-agent reinforcement learning community \cite{FoersterFANW18,RashidSWFFW18,mac2019macca,NEURIPS2019_f816dc0a,qtran2019,qplex2020,denton2020shapley,zhang2023stas}. Yet, with the exception of \citet{kuba2022trust}, all contributions suggest approximate solutions. \citet{kuba2022trust} introduces a clever way of factorizing the joint value function leveraging advantage functions but limits its applicability to trust-region learning algorithms. The sequential-move centralized for decentralized control paradigm applies to all value- and policy-based algorithms. 

Studies conducted by \citet{Cazenave10} and subsequent research by \citet{f63dc22521f74ed399c1f5a52d864d82} demonstrated that one can effectively solve both multi-agent path planning and multi-agent cooperative reinforcement learning problems, one agent at a time. However, \citet{f63dc22521f74ed399c1f5a52d864d82} limited their approach to sequential policies, thereby hindering the ability to derive optimal decentralized policies in general cases. Our research advances this field by applying a broader framework of Dec-POMDP while offering an optimal dynamic programming theory. This paradigm was applied successfully in subclasses of simultaneous-move Dec-POMDPs, namely two and many-player hierarchical information-sharing Dec-POMDPs \cite{xie2020optimally,peralez2024solving}. Similarly, \citet{koops2023recursive} employs a sequential-move centralized training paradigm for decentralized control but confines the analysis to a tree-search algorithm, specifically the multi-agent A* (MAA*) \cite{SzerC06,OliehoekSV08,OliehoekSAW13}. Additionally, \citet{KOVARIK2022103645} demonstrated that simultaneous-move games could be treated as sequential-move problems through extensive-form games, although they did not provide a practical computational theory or algorithms. We posit that transitioning from simultaneous- to sequential-move centralized training for decentralized execution establishes a foundation for enhancing exact and approximate planning and reinforcement learning algorithms applicable to cooperative, competitive, and ultimately mixed-motive partially observable stochastic games. 


\section*{Acknowledgements}
This work was supported by ANR project Planning and Learning to Act in Systems of Multiple Agents under Grant ANR-19-CE23-0018, and ANR project Data and Prior, Machine Learning and Control under Grant ANR-19-CE23-0006, and ANR project Multi-Agent Trust Decision Process for the Internet of Things under Grant ANR-21-CE23-0016, all funded by French Agency ANR. J.C. was supported by SEFRI in the context of the HORIZON Europe Hyper-AI project (grant agreement No. 101135982).


\medskip

{
\bibliography{sample}
}

\newpage
\onecolumn
\appendix

\part*{Supplementary Material}



\section{Preliminaries}

\paragraph{Notation.} 
We denote random variables and their realizations by lowercase letters $y$; their domains are the corresponding uppercase letters $Y$. 
For integers $t_0\leq t_1$ and $i_0\leq i_1$, $y_{t_0:t_1}$ is a short-hand notation for the vector $(y_{t_0}, y_{t_0+1},\ldots,y_{t_1})$ while $y^{i_0:i_1}$ is a short-hand notation for $(y^{i_0}, y^{i_0+1},\ldots,y^{i_1})$, also $y^{i_1:i_0} \doteq \emptyset$. 
Generally, the subscripts are used as control interval indexes, while we use superscripts to index agents. $\Pr\{\cdot\}$ is the probability of an event, and $\mathbb{E}\{\cdot\}$ is the expectation of a random variable. 
For any finite set $\mathcal{Y}$, $|\mathcal{Y}|$ denotes the cardinality of $\mathcal{Y}$, $\mathbb{N}_{\leq |\mathcal{Y}|} \doteq \{0, 1,\ldots, |\mathcal{Y}|\}$,  $\Pr(\mathcal{Y})$ is the $(|\mathcal{Y}|-1)$-dimensional real simplex.
For any function $f\colon \mathcal{Y} \mapsto \mathbb{R}$, $\mathtt{supp}(f)$ denotes the set-theoretic support of $f$, \ie the set of points in $\mathcal{Y}$ where $f$ is non-zero. 
If we let $i_0,i_1 \in \mathbb{N}_{\leq n}$, then we shall use short-hand notations $y^{i_0:}$ for the vector $y^{i_0:n}$ while $y^{:i_1}$ for $y^{0:i_1}$. Finally, we use the Kronecker delta $\delta_{\cdot}(\cdot) \colon \mathcal{Y}_0\times \mathcal{Y}_1 \to \{0,1\}$.

\begin{center}
\begin{table}[!ht]
\caption{Notation for each formulation of Dec-POMDPs used throughout the manuscript.}
\label{table:notation}
\resizebox{.98\textwidth}{!}{
\begin{tabular}{@{}ll r r r r r}

\toprule%
 \centering%
 \bfseries Model &  
 & \bfseries $\mathtt{M}$
 & \bfseries $\mathtt{M}'$
 & \bfseries $M$
 & \bfseries $\pmb{M}$
 & \bfseries $M'$
 \\

\cmidrule[0.4pt](r{0.125em}){1-2}%
\cmidrule[0.4pt](lr{0.125em}){3-3}%
\cmidrule[0.4pt](lr{0.125em}){4-4}%
\cmidrule[0.4pt](l{0.125em}){5-5}%
\cmidrule[0.4pt](l{0.125em}){6-6}%
\cmidrule[0.4pt](l{0.125em}){7-7}%

team size 			&  & 		$n$ 			& 		$n$ 			& 		$n$ 		& 		$n$ 			&		 $n$\\ 
\myrowcolour
state	space		&  & 		$\mathtt{X}$	& 		$\mathtt{S}$	& 		$X$		& 		$\pmb{S}$		&		 $S$\\	
private action space		&  & 		$\mathtt{U}^i$	& 		$\mathtt{A}^i$	& 		$U^i$		& 		$\pmb{A}^i$		&		 $A^i$\\	
\myrowcolour
joint action space		&  & 		$\mathtt{U}$	& 		$\mathtt{A}$	& 		$U$		& 		$\pmb{A}$		&		 $A$\\	
private observation space	&  & 		$\mathtt{Z}^i$	& 		$\emptyset$	& 		$Z^i$		& 		$\emptyset$	&		 $\emptyset$\\	
\myrowcolour
reward model	&  & 		$\mathtt{r}\colon \mathtt{X}\times\mathtt{U}\to \mathbb{R} $	& 		$\mathtt{R}\colon \mathtt{S}\times \mathtt{A}\to \mathbb{R}$	& 		$r\colon X\times U\to \mathbb{R}$		& 		$\pmb{r}\colon \pmb{S}\times \pmb{A}\to \mathbb{R}$	&		 $R\colon S\times A\to \mathbb{R}$\\	
dynamics model	&  & 		$\mathtt{p} = \{\mathtt{p}_{\mathtt{x},\mathtt{y}}^{\mathtt{u},\mathtt{z}}\}$	& 		$\mathtt{T}\colon \mathtt{S}\times \mathtt{A}\to \mathtt{S}$	& 		$p= \{{p}_{{x},{y}}^{{u},{z}}(\tau)\}$		& 		$\pmb{p}\colon \pmb{S}\times \pmb{A}\to \pmb{S}$	&		 $T\colon S\times A\to S$\\	
\myrowcolour
joint observation space	&  & 		$\mathtt{Z}$	& 		$\emptyset$	& 		$Z$		& 		$\emptyset$	&		 $\emptyset$\\	
private history space		&  & 		$\mathtt{O}^i$	& 		$\emptyset$	& 		$O^i$		& 		$\emptyset$	&		 $\emptyset$\\	
\myrowcolour
joint history space		&  & 		$\mathtt{O}$	& 		$\emptyset$	& 		$O$		& 		$\emptyset$	&		 $\emptyset$\\	
planning horizon		&  & 		$\ell$	& 		$\ell$	& 		$\ell'$		& 		$\ell'$	&		 $\ell'$\\	
\myrowcolour
private policy 		&  & 		$\mathtt{a}^i_{0:\ell-1}$	& 		$\mathtt{a}^i_{0:\ell-1}$	& 		$a^i_{0:\ell'-1}$		& 		$a^i_{0:\ell'-1}$	&		 $a^i_{0:\ell'-1}$\\	
joint policy 		&  & 		$\mathtt{a}_{0:\ell-1}$	& 		$\mathtt{a}_{0:\ell-1}$	& 		$a_{0:\ell'-1}$		& 		$ a_{0:\ell'-1}$	&		 $a_{0:\ell'-1}$\\	
\myrowcolour
private decision rule 		&  & 		$\mathtt{a}^i_t\colon \mathtt{O}^i \to \mathtt{U}^i$	& 		$\mathtt{a}^i_t\colon \mathtt{O}^i \to \mathtt{U}^i$	& 		$a_\tau\colon {O}^{\rho(\tau)} \to {U}(\tau)$		& 		$a_\tau\colon {O}^{\rho(\tau)} \to {U}(\tau)$	&		 $a_\tau\colon {O}^{\rho(\tau)} \to {U}(\tau)$\\	
joint decision rule 		&  & 		$\mathtt{a}_t\colon \mathtt{O} \to \mathtt{U}$	& 		$\mathtt{a}_t\colon \mathtt{O} \to \mathtt{U}$	& 		n.a.		& 		n.a.	&		 n.a.\\	
\myrowcolour
state-value function 		&  & 		$\upsilon_{\mathtt{M},\gamma,0}^{\mathtt{a}_{0:\ell-1}}(b_0)$	& 		$\upsilon_{\mathtt{M}',\gamma,0}^{\mathtt{a}_{0:\ell-1}}(\mathtt{s}_0)$	& 		$\upsilon_{M,\lambda,0}^{a_{0:\ell'-1}}(b_0)$		& 		$\upsilon_{\pmb{M},\lambda,0}^{a_{0:\ell'-1}}(\iota_0)$	&		 $\upsilon_{M',\lambda,0}^{a_{0:\ell'-1}}(s_0)$\\	
action-value function 		&  & 		$q_{\mathtt{M},\gamma,0}^{\mathtt{a}_{1:\ell-1}}(b_0,\mathtt{a}_0)$	& 		$q_{\mathtt{M}',\gamma,0}^{\mathtt{a}_{1:\ell-1}}(\mathtt{s}_0,\mathtt{a}_0)$	& 		$q_{M,\lambda,0}^{a_{1:\ell'-1}}(b_0,a_0)$		& 		$q_{\pmb{M},\lambda,0}^{a_{1:\ell'-1}}(\iota_0,a_0)$	&		 $q_{M',\lambda,0}^{a_{1:\ell'-1}}(s_0,a_0)$\\	
\myrowcolour
state-value linear fct. 		&  & 		$\alpha_t\colon \mathtt{X}\times \mathtt{O}\to \mathbb{R}$	& 		$\alpha_t\colon \mathtt{X}\times \mathtt{O}\to \mathbb{R}$	& 		$\alpha_\tau\colon X_\tau\times O\to \mathbb{R}$		& 		$\alpha_\tau\colon X_\tau\times O\to \mathbb{R}$	&		 $\alpha_\tau\colon X_\tau \times O\to \mathbb{R}$\\	
action-value linear fct. 		&  & 		$\beta_t\colon \mathtt{X}\times \mathtt{O}\times \mathtt{U}\to \mathbb{R}$	& 		$\beta_t\colon \mathtt{X}\times \mathtt{O}\times \mathtt{U}\to \mathbb{R}$	& 		$\beta_\tau\colon X_\tau \times O\times U(\tau)\to \mathbb{R}$		& 		$\beta_\tau\colon X_\tau \times O\times U(\tau)\to \mathbb{R}$	&		 $\beta_\tau\colon X_\tau \times O\times U(\tau)\to \mathbb{R}$\\	

\bottomrule
\end{tabular}
}
\end{table}
\end{center}

\section{Sequential-Move Single-Agent Reformulation}

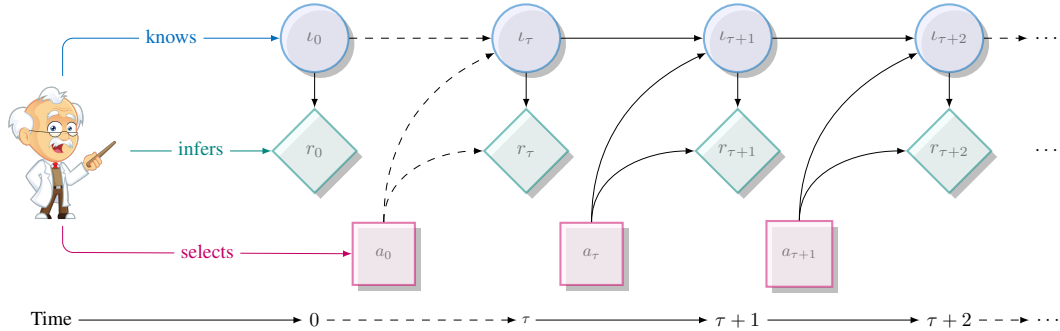
\begin{figure}[!ht]
 \centering
 \begin{tikzpicture}[->,-latex,auto,node distance=3.75cm,semithick, square/.style={regular polygon,regular polygon sides=4}] %
  \tikzstyle{every state}=[draw=black,text=black,inner color= white,outer color= white,draw= black,text=black, drop shadow]
  \tikzstyle{place}=[thick,draw=sthlmBlue,fill=blue!20,minimum size=12mm, opacity=.5]
  \tikzstyle{red place}=[square,place,draw=sthlmRed,fill=sthlmLightRed,minimum size=18mm]
  \tikzstyle{green place}=[diamond,place,draw=sthlmGreen,fill=sthlmLightGreen,minimum size=15mm]


  \node[fill=white, scale=.75] (T) at (-3.5,-3.75) {};
  \node[fill=white, scale=.75] (T0) at (0,-3.75) {$0$};
  \node[fill=white, scale=.75]     (T1) [right of=T0,node distance=3.75cm, fill=white, scale=.75] {$\tau$};
  \node[fill=white, scale=.75]     (T2) [right of=T1,node distance=3.75cm] {$\tau+1$};
  \node[fill=white, scale=.75]     (T3) [right of=T2,fill=white, node distance=3.75cm] {$\tau+2$};
  \node[fill=white, scale=.75]     (T4) [right of=T3,fill=white, node distance=1.75cm] {$\ldots$};
  \draw[->,-latex,dashed,very thin, color=black,anchor=mid] (T3) -- (T4); 
  \draw[->,-latex, very thin, color=black, anchor=mid] (T2) -- (T3); 
  \draw[->,-latex, very thin, color=black, anchor=mid] (T1) -- (T2); 
  \draw[->,-latex,dashed,very thin, color=black,anchor=mid] (T0) -- (T1); 
  \draw[->,-latex,rounded corners,very thin, color=black,anchor=mid] (T) node[fill=white, scale=.75]{Time} -- (T0); 
 %

 \node[state,place, scale=.75] (S0)            {$\iota_0$};
  \node[state,place, scale=.75]     (S1) [right of=S0] {$\iota_\tau$};
  \node[state,place, scale=.75]     (S2) [ right of=S1] {$\iota_{\tau+1}$};
  \node[state,place, scale=.75]     (S3) [ right of=S2] {$\iota_{\tau+2}$};
  \node[, scale=.75]     (S4) [ right of=S3,node distance=1.75cm] 		 {$\cdots$};

 \path[very thin] (S0) edge[dashed] node[midway,text=black,draw=none,scale=.7, above=-5pt] 		{} (S1) 
      (S1) edge node[midway,text=black,draw=none,scale=.7, above=-5pt] 		{} (S2) 
      (S2) edge node[midway,text=black,draw=none,scale=.7, above=-5pt] 		{} (S3) 
      (S3) edge[dashed] (S4);

 \node[, scale=.75]     (A_SLACK) [below of=S0,node distance=2.6cm] {};

 \node[state,red place, scale=.65]     (A0) [below right of=A_SLACK,node distance=2cm] {$a_{0}$};
 \node[state,red place, scale=.65]     (A1) [right of=A0,node distance=4.25cm]    {$a_\tau$};
 \node[state,red place, scale=.65]     (A2) [right of=A1,node distance=4.25cm]    {$a_{\tau+1}$};
 \node[, scale=.75]     (A3) [right of=A2,node distance=5.5cm]    {};

 \path[very thin] (A0) edge [out=90, in=-155, dashed] node {} (S1)
     (A1) edge [out=90, in=-155] node {} (S2)
     (A2) edge [out=90, in=-155] node {} (S3);   


 \node[state,green place, scale=.75]     (O0) [below of=S0,node distance=2cm] {$r_0$};
 \node[state,green place, scale=.75]     (O1) [below of=S1,node distance=2cm] {$r_\tau$};
 \node[state,green place, scale=.75]     (O2) [below of=S2,node distance= 2cm] {$r_{\tau+1}$};
 \node[state,green place, scale=.75]     (O3) [below of=S3,node distance= 2cm] {$r_{\tau+2}$};
  \node[, scale=.75]     (O4) [ right of=O3,node distance=1.75cm] 		 {$\cdots$};

  \node[inner sep=0pt] (prof) at (-3.35,-1.5) {\includegraphics[width=.1\textwidth]{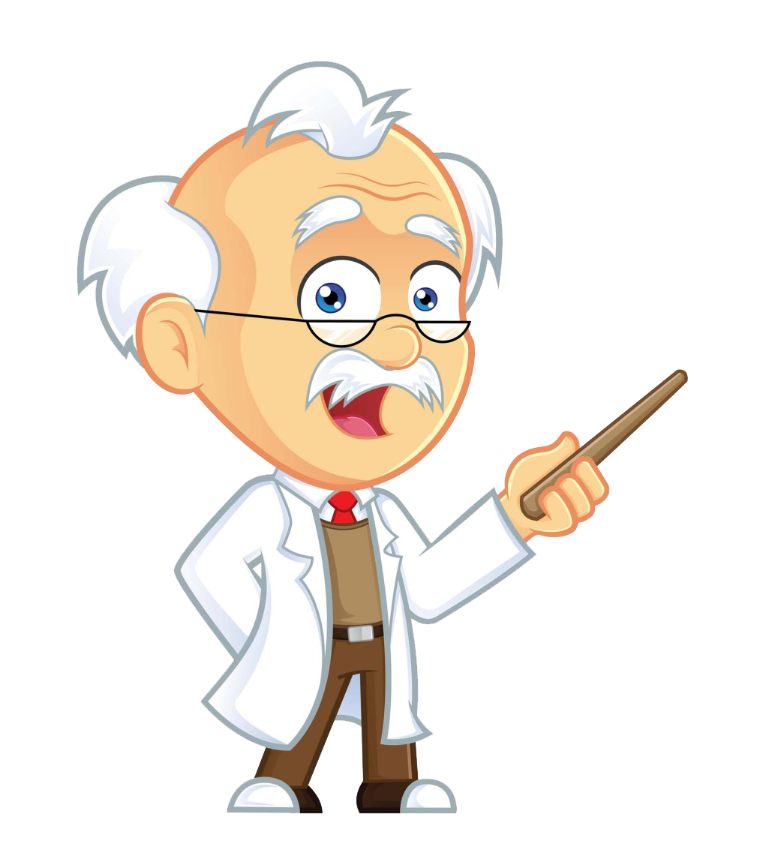}};
  \draw[->,-latex,rounded corners,very thin, color=sthlmBlue,anchor=mid] (prof) -- (-3.35,0) --node[draw=none, fill=white, scale=.75] {knows} (S0);	
  \draw[->,-latex,rounded corners,very thin, color=sthlmGreen,anchor=mid] (prof) --node[draw=none, fill=white, scale=.75] {infers} (O0); 	
  \draw[->,-latex,rounded corners,very thin, color=sthlmRed,anchor=mid] (prof) -- (-3.35,-2.85) --node[draw=none, fill=white, scale=.75] {selects} (A0);	

 \path[very thin] (S0) edge  node {} (O0)
 	 (S1) edge  node {} (O1)
      (A0) edge[out=90, in=-180, dashed] node {} (O1)
  	(S2) edge  node {} (O2)
     (A1) edge[out=90, in=-180] node {} (O2)
  	(S3) edge  node {} (O3)
     (A2) edge[out=90, in=-180] node {} (O3);
 \end{tikzpicture}
 \caption{Influence diagram for a $2$-agent sequential-move Markov decision process.}
 \label{fig:emdp}
 \end{figure}

\section{On Sufficiency of Sequential-Move Occupancy States}

\subsection{Proof of Lemma \ref{lem:sequence:approach}}
\label{prooflemsequenceapproach}

\lemsequenceapproach*
\begin{proof}
For every $t\in  \mathbb{N}_{\leq \ell-1}$ and any arbitrary $t$-th (simultaneous) occupancy state 
$\mathtt{s}_{t}$
\begin{align*}
\upsilon_{\mathtt{M'},\gamma,t}^*(\mathtt{s}_{t}) & \overset{(\ref{eq:bellman})}{=} {\textstyle\max_{\mathtt{a}_{t}\in \mathtt{{A}}_t} ~\left[\mathtt{r}(\mathtt{s}_{t},\mathtt{a}_{t}) + \gamma \upsilon_{\mathtt{M'},\gamma,t+1}^*(\mathtt{p}(\mathtt{s}_{t},\mathtt{a}_{t}))\right]}\\
&= {\textstyle\max_{\mathtt{a}_{t}^1\in \mathtt{A}_{t}^1,\cdots,\mathtt{a}_{t}^{n}\in \mathtt{A}_{t}^{n}}~\left[\mathtt{r}(\mathtt{s}_{t},\mathtt{a}_{t}) + \gamma \upsilon_{\mathtt{M'},\gamma,t+1}^*(\mathtt{p}(\mathtt{s}_{t},\mathtt{a}_{t}))\right]}\\
&= {\textstyle \max_{\mathtt{a}_{t}^1\in \mathtt{A}_t^1} \cdots \max_{ \mathtt{a}_{t}^{n} \in \mathtt{A}_{t}^{n}}~ \left[ \mathtt{r}(\mathtt{s}_{t},\mathtt{a}_{t})+ \gamma \upsilon_{\mathtt{M'} ,\lambda,t+1}^*(\mathtt{p}(\mathtt{s}_{t},\mathtt{a}_{t})) \right]} \\
&= {\textstyle \max_{ \mathtt{a}_{t}^{\sigma(1)} \in \mathtt{A}_{t}^{\sigma(1)} } \cdots \max_{ \mathtt{a}_{t}^{\sigma(n)}\in \mathtt{A}_{t}^{\sigma(n)} }~ \left[ \mathtt{r}(\mathtt{s}_{t},\mathtt{a}_{t}) + \gamma \upsilon_{\mathtt{M'},\gamma, t+1}^*(\mathtt{p}(\mathtt{s}_{t},\mathtt{a}_{t})) \right]}
\end{align*}
where $\sigma\colon \mathbb{N}^*_{\leq n} \to \mathbb{N}^*_{\leq n}$ is any permutation over agents. 
The last equality holds since switching whatsoever the positions of the $\max$ operators does not affect the outcome.
\end{proof}

\subsection{Proof of Theorem \ref{thm:simultaneous:sequential}}
\label{proofthmsimultaneoussequential}

\thmsimultaneoussequential*  
\begin{proof}
The proof of this statement is threefold. 
We first prove any simultaneous decentralized partially observable Markov decision process $\mathtt{M}$ exhibits a corresponding sequential decentralized partially observable Markov decision process $M$.
Next, we demonstrate one can always recast every sequential joint policy $\pi$ into a simultaneous one $\xi$.
Finally, we show optimal simultaneous and sequential joint policies $\xi^*$ and $\pi^*$ for $\mathtt{M}$ and $M$ yield the same value.

\paragraph{Building $M$ based upon $\mathtt{M}$.} 
The process $\mathtt{M}$ describes evolves over $\ell$ decision epochs $t\in \mathbb{N}_{\leq \ell-1}$, each of which allow all $n$ agents to act simultaneously. 
$\mathtt{M}$ can alternative be viewed as to process $M$ that evolves over $n \ell$ points in time, each of which allow only a single agent to act. 
More specifically, every decision epoch $t$ in $\mathtt{M}$ corresponds to $n$ points in time $\tau \in \{n t, n t+1, \ldots, n t + n -1 \}$, where agent $i$ acts at point in time $\tau = n t + i - 1$.
Hence, for every agent $i\in \mathbb{N}^*_{\leq n}$, the set of decision points is given by  $H^i \doteq \{n t + i -1 \mid t\in \mathbb{N}_{\leq \ell-1} \}$, thus $\ell' \doteq n\ell$. 
Every point in time $\tau = nt + i - 1$, the sequential process $M$ moves from one sequential state $x_{\tau}$ to another one $x_{\tau+1}$ at the subsequent point in time; 
yet, the simultaneous state $\mathtt{x}_t$ may be fixed depending on $i$.
If agent $i$ is the last to act in decision epoch $t$, \ie $i = n$, then sequential and simultaneous states are identical $x_{\tau+1} \doteq \mathtt{x}_{t+1}$.
Otherwise, sequential state $x_{\tau+1}$ at the subsequent point in time is the previous state $x_\tau$ along with the action agent $i$ took $u_\tau$ available in finite action set $U(\tau)$---\ie $x_{\tau+1} \doteq (x_\tau, u_\tau)$.  
Overall, sequential state and action spaces are given by $X \doteq \mathtt{X} \cup_{i\in \mathbb{N}^*_{\leq n}}~(\mathtt{X}\times \mathtt{U}^{1:i})$ and $U(\tau) \doteq \mathtt{U}^{\rho(\tau)}$, respectively.
Upon the execution of action $u_\tau$ at point in time $\tau$, process $\mathtt{M}$ responds differently according to agent $i$.
If $i=n$ then $x_\tau\doteq(\mathtt{x}_t,\mathtt{u}^{1:n-1}_t)$, $u_\tau\doteq\mathtt{u}^n_t$.
$\mathtt{M}$ immediately responds with common reward $r_{x_\tau,u_\tau}(\tau) \doteq \mathtt{r}_{\mathtt{x}_t,\mathtt{u}_t}$, and responds in the subsequent point in time by  moving into a new simultaneous state $x_{\tau+1} \doteq \mathtt{x}_{\tau+1}$ and gives sequential observation $z_{\tau+1}^j \doteq \mathtt{z}^j_{t+1}$ for each agent $j\in \mathbb{N}^*_{\leq n}$ according to probability 
\begin{align*}
    p_{x_\tau, x_{\tau+1}}^{u_\tau,z_{\tau+1}}(\tau) &\doteq \mathtt{p}_{\mathtt{x}_t,\mathtt{x}_{t+1}}^{\mathtt{u}_t,\mathtt{z}_{t+1}}.
\end{align*} 
Otherwise, \ie if $i\neq n$, $\mathtt{M}$ does not provides any reward, \ie $r_{x_\tau,u_\tau}(\tau) \doteq 0$ and it does not change the current 
simultaneous state, \ie $x_{\tau+1} \doteq (x_\tau, u_\tau)$, nor it gives observations to agents, \ie $z_{\tau+1}^j \doteq \emptyset$ for all agents $j\in \mathbb{N}^*_{\leq n}$, with probability
\begin{align*}
p^{u_\tau,z_{\tau+1}}_{x_\tau,x_{\tau+1}}(\tau) &\doteq \delta_{\emptyset}(z_{\tau+1}) \cdot \delta_{(x_\tau,u_\tau)}(x_{\tau+1}).
\end{align*}
Optimally solving the resulting $M \doteq \langle \ell',X,U,Z,\rho,p,r \rangle$ aims at finding a sequential joint policy $\pi$, \ie a sequence of private decision rules, 
$\pi \doteq (a_\tau)_{\tau\in \mathbb{N}_{\leq \ell'-1}}$,
which maximizes the expected $\lambda$-discounted cumulative reward starting at initial state distribution $b_0$ onward, 
and given by 
$
  \upsilon^\pi_{M,\lambda,0}(b_0) \doteq \mathbb{E}\{ \textstyle{\sum_{\tau=0}^{\ell'-1}}~\lambda(\tau)\cdot r_{x_{\tau},u_{\tau}}(\tau) ~|~b_0, \pi\}$,
where $\lambda(\tau) \doteq \gamma^t$.

\paragraph{Policy class equivalence.} Here, we show that by solving $M$ we do not focus onto a restricted class of policies. 
Before proceeding any further, recall that $U(\tau) \doteq \mathtt{U}^{\rho(\tau)}$, which 
means private actions of agent $i=\rho(\tau)$ lives in the same finite set.
Moreover, sequential private observations of agent $i$ live in finite set $Z^i \doteq \mathtt{Z}^i\cup \{\emptyset\}$.
If we discard empty sequential observations from private sequential observations of agent $i$, then we know private sequential history $o^i_\tau \in  \mathtt{O}_t^i$ since $o^i_\tau$ is consists of $t$ actions from $\mathtt{U}^i$ and observations from $\mathtt{Z}^i$, as does any private simultaneous history $\mathtt{o}_t^i$. 
In other words, one can always go from $\xi$ to $\pi$ and vice versa. 
To better understand this, notice that $\pi^i \doteq (a_\tau)_{\tau\in H^i}$ describes a sequence of private sequential decision rule of agent $i$, one such sequential decision rule $a_\tau$ for each decision epoch $t\in \mathbb{N}_{\leq \ell-1}$. 
Moreover, each sequential decision rule $a_\tau\colon \mathtt{O}_t^i \mapsto \mathtt{U}^i$ maps private sequential (resp. simultaneous) histories to private sequential (resp. simultaneous) actions.  
As a consequence, for every $a_\tau$ (resp. $\pi^*$) there exists an equivalent $\mathtt{a}_t^i\colon \mathtt{O}_t^i \to \mathtt{U}^i$ (resp. $\xi^*$).  

\paragraph{Performance transfer from $M$ to corresponding $\mathtt{M}$.}
We start with the formal definition of optimal sequential joint policy $\pi^*$, 
$\upsilon^{\pi^*}_{M,\lambda,0}(b_0) \doteq \mathbb{E}\{ \textstyle{\sum_{\tau=0}^{\ell'-1}}~\lambda(\tau)\cdot r_{x_{\tau},u_{\tau}}(\tau) ~|~b_0, \pi^*\}$.
Next, we substitute $\lambda(\tau)$ by $\gamma^t$, \ie $\upsilon^{\pi^*}_{M,0}(b_0) = \mathbb{E}\{ \textstyle{\sum_{\tau\in {T}^n}}~\gamma^t\cdot r_{x_{\tau},u_{\tau}}(\tau) ~|~b_0, \pi^*\}$.
We then exploit the fact that immediate reward for $i\neq n$ are zero, \ie non-zero rewards $r_{x_{\tau},u_{\tau}}(\tau) \doteq \mathtt{r}_{\mathtt{x}_t,\mathtt{u}_t}$  occur only in point set ${T}^n$, 
$\upsilon^{\pi^*}_{M,\lambda,0}(b_0) = \mathbb{E}\{ \textstyle{\sum_{\tau\in {T}^n}}~\gamma^t\cdot \mathtt{r}_{\mathtt{x}_t,\mathtt{u}_t} ~|~b_0, \pi^*\}$ where $\tau \doteq nt$.
Finally, exploiting expression $\tau \doteq nt$, we obtain the following relation since the dynamics are identical over $H^n$, \ie
\begin{align*}
\upsilon^{\pi^*}_{M,\lambda,0}(b_0) &=
\mathbb{E}\{ \textstyle{\sum_{t\in \mathbb{N}_{\leq \ell-1}}}~\gamma^t\cdot \mathtt{r}_{\mathtt{x}_t,\mathtt{u}_t} ~|~b_0, \pi^*
\} 
\doteq \upsilon^{\pi^*}_{\mathtt{M},\gamma,0}(b_0).
\end{align*}
Taking altogether the completeness of sequential joint policy space for simultaneous multi-agent problems and equivalence in performance evaluation for both policies leads to the statement.
\end{proof}

\subsection{Proof of Lemma \ref{lem:equivalence:sequential}}
\label{prooflemequivalencesequential}

\lemequivalencesequential*
\begin{proof}
In demonstrating the proof, we must show that optimal sequential joint policies $\pi^*_{\pmb{M}}$ and $\pi^*_M$ achieve the same $\lambda$-discounted cumulative reward.
The proof starts with the definition of an optimal sequential joint policy $\pi^*_{\pmb{M}}$, \ie
\begin{align*}
    \pi^*_{\pmb{M}} &\doteq {\textstyle \argmax_{(a_\tau)_{\tau\in \mathbb{N}_{\leq \ell'-1}}}~\sum_{\tau\in \mathbb{N}_{\leq \ell'-1}}}~\pmb{r}(\iota_\tau, a_\tau). 
\end{align*}
Next, we replace $\pmb{r}(\iota_\tau, a_\tau)$ by the immediate expected rewards received after taking sequential action $a_\tau$ at sequential exhaustive data $\iota_\tau$, \ie
\begin{align*}
    \pi^*_{\pmb{M}} &= {\textstyle \argmax_{(a_\tau)_{\tau\in \mathbb{N}_{\leq \ell'-1}}}~\sum_{\tau\in \mathbb{N}_{\leq \ell'-1}}}~ \mathbb{E}\{ \lambda(\tau)\cdot r_{x_\tau,u_\tau}(\tau) \mid \iota_0, (a_\tau)_{\tau\in \mathbb{N}_{\leq \ell'-1}} \}. 
\end{align*}
The linearity of the expectation enables us to rewrite the previous expression as follows, \ie  
\begin{align*}
    \pi^*_{\pmb{M}} &= {\textstyle \argmax_{(a_\tau)_{\tau\in \mathbb{N}_{\leq \ell'-1}}}~\mathbb{E}\{\sum_{\tau\in \mathbb{N}_{\leq \ell'-1}}~   \lambda(\tau)\cdot r_{x_\tau,u_\tau}(\tau) \mid \iota_0, (a_\tau)_{\tau\in \mathbb{N}_{\leq \ell'-1}} \}} \doteq \pi^*_M. 
\end{align*}
Hence, searching for an optimal sequential joint policy using sequential data-driven Markov decision process $\pmb{M}$ is sufficient to find an optimal sequential joint policy for corresponding sequential decentralized partially observable Markov decision process $M$ (and vice versa).
\end{proof}

\subsection{Proof of Lemma \ref{lem:markov}} 
\label{prooflemmarkov}

\lemmarkov*
\begin{proof}
In demonstrating this statement, we also derive a procedure for updating sequential occupancy states.
Let $\iota_{\tau+1}$ be the sequential exhaustive data available at sequential point in time $\tau+1$, which we rewrite as $\iota_{\tau+1} \doteq (\iota_\tau, a_\tau)$, that is sequential exhaustive data $\iota_\tau$ plus chosen sequential decision rule $a_\tau$ at sequential point in time $\tau$.
By Definition \ref{def:soc}, we relate both sequential occupancy state and exhaustive data as follows:
for every sequential state $x_{\tau+1}$ and joint history $o_{\tau+1}$,
\begin{align}
 s_{\tau+1}(x_{\tau+1}, o_{\tau+1}) %
 &=
 \Pr\{x_{\tau+1}, o_{\tau+1} \mid \iota_{\tau+1}\}.
 \label{eq:sequentialoccupancystate:1}
\end{align}
The substitution of $\iota_{\tau+1}$ by $(\iota_\tau, a_\tau)$ into  (\ref{eq:sequentialoccupancystate:1}) yields
\begin{align}
 &\overset{(\ref{eq:sequentialoccupancystate:1})}{=}
 \Pr\{x_{\tau+1}, o_{\tau+1} \mid \iota_\tau, a_\tau\}.
 \label{eq:sequentialoccupancystate:2}
\end{align}
The expansion of the right-hand side of (\ref{eq:sequentialoccupancystate:2}) over sequential states of $M$ at point in time $\tau$ gives 
\begin{align}
 &\overset{(\ref{eq:sequentialoccupancystate:1})}{=}
 {\textstyle \sum_{x_\tau\in X}}~
 \Pr\{x_\tau, x_{\tau+1}, o_\tau, u_\tau, z_{\tau+1} \mid \iota_\tau, a_\tau\}.
 \label{eq:sequentialoccupancystate:3}
\end{align}
The expansion of the joint probability into the product of conditional probabilities leads to 
\begin{align}
 &\overset{(\ref{eq:sequentialoccupancystate:1})}{=}
 {\textstyle \sum_{x_\tau\in X}}~
 \Pr\{x_{\tau+1},  z_{\tau+1} \mid x_\tau, o_\tau, u_\tau, \iota_\tau, a_\tau\}  
\cdot  \Pr\{u_\tau \mid x_\tau,  o_\tau, \iota_\tau, a_\tau\}
 \cdot \Pr\{x_\tau, o_\tau \mid \iota_\tau, a_\tau\}.
 \label{eq:sequentialoccupancystate:4}
\end{align}
The first factor on the right-hand side of (\ref{eq:sequentialoccupancystate:4}) is the dynamics model in $M$, \ie 
\begin{align}
 &\overset{(\ref{eq:sequentialoccupancystate:1})}{=}
 {\textstyle \sum_{x_\tau\in X}}~
  p^{u_\tau, z_{\tau+1}}_{x_\tau,x_{\tau+1}}(\tau)  
 \cdot  \Pr\{u_\tau \mid x_\tau,  o_\tau, \iota_\tau, a_\tau\}
 \cdot \Pr\{x_\tau, o_\tau \mid \iota_\tau, a_\tau\}.
 \label{eq:sequentialoccupancystate:5}
\end{align}
The second factor in the right-hand side of (\ref{eq:sequentialoccupancystate:5}) denotes the sequential action $u_\tau$ that sequential decision rule $a_\tau$ prescribe to agent $\rho(\tau)$ at sequential private history $o_\tau^{\rho(\tau)}$, \ie 
\begin{align}
 &\overset{(\ref{eq:sequentialoccupancystate:1})}{=}
 {\textstyle \sum_{x_\tau\in X}}~
  p^{u_\tau, z_{\tau+1}}_{x_\tau,x_{\tau+1}}(\tau)  
 \cdot a_\tau( u_\tau \mid o^{\rho(\tau)}_\tau )
 \cdot \Pr\{x_\tau, o_\tau \mid \iota_\tau, a_\tau\}.
 \label{eq:sequentialoccupancystate:6}
\end{align}
The last factor in the right-hand side of (\ref{eq:sequentialoccupancystate:6}) defines the prior sequential occupancy state $s_\tau$, \ie %
\begin{align}
 &\overset{(\ref{eq:sequentialoccupancystate:1})}{=}
 {\textstyle \sum_{x_\tau\in X}}~
 p^{u_\tau, z_{\tau+1}}_{x_\tau,x_{\tau+1}}(\tau) 
 \cdot a_\tau( u_\tau \mid o^{\rho(\tau)}_\tau )
 \cdot s_\tau(x_\tau, o_\tau).
 \label{eq:sequentialoccupancystate:7}
\end{align}
Therefore, the update of the current sequential occupancy state $s_{\tau+1}$ depends solely upon the prior sequential occupancy state $s_\tau$ and the chosen sequential decision rule $a_\tau$.
\end{proof}

\subsection{Proof of Lemma \ref{lem:reward}} 
\label{prooflemreward} 

\lemreward*
\begin{proof}
The proof directly builds on Lemma \ref{lem:markov}. 
In demonstrating the property, we need to show that the probability of the process $M$ being in sequential state $x_\tau$ and taken sequential action being $u_\tau$ , at the end of sequential point in time $t\in \mathbb{N}_{\leq \ell'-1}$, depends on sequential joint history $o_\tau$ only through the corresponding sequential occupancy state $s_\tau \doteq (\Pr\{x_\tau, o_\tau \mid \iota_\tau \})_{x_\tau,o_\tau}$. 
To this end, we start with the definition of expected immediate reward $\pmb{r}(\iota_\tau, a_\tau)$, \ie
\begin{align}
    \pmb{r}(\iota_\tau, a_\tau) &\doteq
    \mathbb{E}\{  \lambda(\tau)\cdot r_{x_\tau,u_\tau}(\tau) \mid \iota_\tau, a_\tau \}.
    \label{eq:lemreward:1}
\end{align}
The expansion of the expectation over sequential actions, states, and joint histories gives
\begin{align}
    \pmb{r}(\iota_\tau, a_\tau) &=
    {\textstyle\sum_{x_\tau,u_\tau, o_\tau}}~
     \lambda(\tau)\cdot r_{x_\tau,u_\tau}(\tau) \cdot
    a_\tau(u_\tau \mid o^{\rho(\tau)}_\tau) 
    \cdot  \Pr\{ x_\tau, o_\tau \mid \iota_\tau\}.
    \label{eq:lemreward:2}
\end{align}
The substitution of $s_\tau$ by $(\Pr\{x_\tau, o_\tau \mid \iota_\tau \})_{x_\tau,o_\tau}$ into the right-hand side of (\ref{eq:lemreward:2}) results in 
\begin{align}
    \pmb{r}(\iota_\tau, a_\tau) &=
    {\textstyle\sum_{x_\tau,u_\tau, o_\tau}}~
     \lambda(\tau)\cdot r_{x_\tau,u_\tau}(\tau) \cdot
    a_\tau(u_\tau \mid o^{\rho(\tau)}_\tau) 
    \cdot  s_\tau( x_\tau, o_\tau ).
    \label{eq:lemreward:3}
\end{align}
By the definition of expected immediate reward $r(s_\tau, a_\tau)$, the right-hand side of (\ref{eq:lemreward:3}) becomes
\begin{align}
    \pmb{r}(\iota_\tau, a_\tau) &=
    r(s_\tau,a_\tau).
    \label{eq:lemreward:4}
\end{align}
Which ends the proof.
\end{proof}

\subsection{Proof of Theorem \ref{thm:sufficiency:condition}}
\label{proofthmsufficiencycondition} 

\thmsufficiencycondition*
\begin{proof}
To demonstrate the statement, we only need to show that by replacing sequential exhaustive data by corresponding sequential occupancy states we achieve the same optimal value functions $(\upsilon^*_{M',\lambda,\tau})_{\tau\in \mathbb{N}_{\leq \ell'-1}}$ and $(\upsilon^*_{\pmb{M},\lambda,\tau})_{\tau\in \mathbb{N}_{\leq \ell'-1}}$  for sequential occupancy Markov decision process $M'$ and original sequential data-driven Markov decision process $\pmb{M}$ (and the corresponding sequential decentralized partially observable Markov decision process $M$), respectively. 
The proof proceeds by induction.
The statement trivially holds at the last point in time $\tau=\ell'$, \ie for any arbitrary sequential exhaustive data and corresponding sequential occupancy state $\iota_{\ell'}$ and $s_{\ell'} \doteq (\Pr\{x_{\ell'},o_{\ell'} \mid \iota_{\ell'}\})_{x_{\ell'},u_{\ell'}}$, respectively, 
\begin{align}
    \upsilon^*_{M',\lambda,\ell'}(s_{\ell'}) &\doteq 0 \doteq \upsilon^*_{\pmb{M},\lambda,\ell'}(\iota_{\ell'})  
    \label{eq:thmsufficiencycondition:1}
\end{align}
Suppose that the statement holds for point in time $\tau+1$, we can now show it also holds for point in time $\tau$.
\citeauthor{bellman}'s optimality criterion prescribes the following system of equations: for every point in time $\tau$ and every sequential exhaustive date $\iota_\tau$, 
\begin{align}
    \upsilon^*_{\pmb{M},\lambda,\tau}(\iota_\tau) &=
    {\textstyle \max_{a_\tau\in A_\tau}}~
    \left[ 
     \pmb{r}(\iota_\tau, a_\tau) + 
    \upsilon^*_{\pmb{M},\lambda,\tau+1}(\pmb{p}(\iota_\tau, a_\tau))
    \right]
    \label{eq:thmsufficiencycondition:2}
\end{align}
By the induction hypothesis, if we let sequential occupancy states  
\begin{align}
s_\tau &\doteq (\Pr\{x_\tau,o_\tau \mid \iota_\tau\})_{x_\tau,o_\tau}
\label{eq:thmsufficiencycondition:3}
\\   
p(s_\tau, a_\tau) &\doteq (\Pr\{x_{\tau+1},o_{\tau+1} \mid \pmb{p}(\iota_\tau, a_\tau) \})_{x_{\tau+1},o_{\tau+1}}   
\nonumber
\end{align} 
then we have that $\upsilon^*_{\pmb{M},\lambda,\tau+1}(\pmb{p}(\iota_\tau, a_\tau)) = \upsilon^*_{\pmb{M},\lambda,\tau+1}(p(s_\tau, a_\tau))$.
Hence, Equation (\ref{eq:thmsufficiencycondition:2}) becomes 
\begin{align}
    \upsilon^*_{\pmb{M},\lambda,\tau}(\iota_\tau) &=
    {\textstyle \max_{a_\tau\in A_\tau}}~
    \left[ 
    \pmb{r}(\iota_\tau, a_\tau) + 
    \upsilon^*_{\pmb{M},\lambda,\tau+1}(p(s_\tau, a_\tau))
    \right]
    \label{eq:thmsufficiencycondition:5}
\end{align}
By Lemma \ref{lem:reward} along with Equation (\ref{eq:thmsufficiencycondition:3}), one can substitute $\pmb{r}(\iota_\tau, a_\tau)$ by $r(s_\tau, a_\tau)$ in the right-hand side of Equation  (\ref{eq:thmsufficiencycondition:5}), \ie 
\begin{align}
    \upsilon^*_{\pmb{M},\lambda,\tau}(\iota_\tau) &=
    {\textstyle \max_{a_\tau\in A_\tau}}~
    \left[ 
   r(s_\tau, a_\tau) + 
    \upsilon^*_{\pmb{M},\lambda,\tau+1}(p(s_\tau, a_\tau))
    \right]
    \label{eq:thmsufficiencycondition:6}
\end{align}
The right-hand side of Equation (\ref{eq:thmsufficiencycondition:6}) describes the optimal value at sequential occupancy state $s_\tau$, \ie 
\begin{align}
    \upsilon^*_{\pmb{M},\lambda,\tau}(\iota_\tau) &=
    {\textstyle \max_{a_\tau\in A_\tau}}~
    \left[ 
    r(s_\tau, a_\tau) + 
    \upsilon^*_{\pmb{M},\lambda,\tau+1}(p(s_\tau, a_\tau))
    \right]
    = \upsilon^*_{M',\lambda,\tau}(s_\tau).
    \label{eq:thmsufficiencycondition:7}
\end{align}
Consequently, the statement holds for every sequential point in time $\tau\in \mathbb{N}_{\leq \ell'-1}$ .
The statement further holds for sequential decentralized partially observable Markov decision process $M$ by Lemma \ref{lem:equivalence:sequential}, \ie an optimal sequential joint policy for $\pmb{M}$ is also an optimal sequential joint policy for $M$.
\end{proof}

\section{Exhibiting Piecewise Linearity and Convexity}

\subsection{Bilinearity of Reward Functions}
\label{prooflemlinearreward}

\begin{restatable}[]{lemma}{lemlinearreward}
\label{lem:linear:reward}
For any sequential occupancy Markov decision process $M'$, expected immediate reward $r\colon S\times A\to \mathbb{R}$ is bilinear \wrt sequential occupancy states and stochastic sequential decision rules.
Furthermore, there exists an equivalent reward function $\hat{r} \colon \hat{S}\odot\hat{A} \to \mathbb{R}$---\ie 
$    
    \hat{r}\colon  \hat{s}_\tau\odot\hat{a}_\tau \mapsto
    {\textstyle \sum_{x_\tau,u_\tau,o_\tau}}~
    (\hat{s}_\tau\odot\hat{a}_\tau)(x_\tau,u_\tau,o_\tau)\cdot
   \lambda(\tau)\cdot r_{x_\tau,u_\tau}(\tau)
$
which is linear in state-action occupancy measures.
\end{restatable}
\begin{proof}
The proof directly follows from the definitions of the expected immediate reward functions $r$, equivalent variant $\hat{r}$ and sequential state-action occupancy measures.
\end{proof}

\subsection{Bilinearity of Transitions}
\label{prooflemlineartransition}

\begin{restatable}[]{lemma}{lemlineartransition}
\label{lem:linear:transition}
For any sequential occupancy Markov decision process $M'$, transition function $p\colon S\times A\to S$ is bilinear \wrt sequential occupancy states and stochastic sequential decision rules.
Furthermore, there exists an equivalent transition function $\hat{p} \colon \hat{S}\odot\hat{A} \to S$---\ie 
$
    \hat{p}\colon  (\hat{s}_\tau\odot\hat{a}_\tau)(x_{\tau+1},\langle o_\tau,u_\tau,z_{\tau+1} \rangle)
    \mapsto
    {\textstyle \sum_{x_\tau}}~
    (\hat{s}_\tau\odot\hat{a}_\tau)(x_\tau,u_\tau,o_\tau)
    \cdot p_{x_\tau,x_{\tau+1}}^{u_\tau,z_{\tau+1}}(\tau)
$
which is linear in sequential state-action occupancy measures.
\end{restatable}
\begin{proof}
The proof directly follows from the definitions of the transition function $p$, equivalent variant $\hat{p}$ and sequential state-action occupancy measures.
\end{proof}

\subsection{Bilinearity of State-Value Functions Under a Policy}
\label{prooflemlinearpolicyvaluefct}

\begin{restatable}[]{lemma}{lemlinearpolicyvaluefct}
\label{lem:linear:policy:value:fct}
For any sequential occupancy Markov decision process $M'$, state-value function $\upsilon_{M',\lambda,0:\ell'}^\pi\colon S\to \mathbb{R}^{\ell'+1}$ is linear \wrt sequential occupancy states.
\end{restatable}
\begin{proof}
The proof proceeds by induction.
The statement trivially holds at the last point in time $\tau=\ell'$, \ie for every sequential occupancy state $s_{\ell'}$,
$\upsilon_{M',\lambda,\ell'}^\pi(s_{\ell'}) \doteq 0$
being constant, it is also linear.
Suppose now that the statement holds at point in time $\tau+1$, we can demonstrate it also holds for point in time $\tau$. 
The application of \citeauthor{bellman} equations at point in time $\tau$ gives the following equations 
\begin{align}
    \upsilon_{M',\lambda,\tau}^\pi(s_\tau) 
    &= 
      R(s_\tau,a_\tau) +
    \upsilon_{M',\lambda,\tau+1}^\pi(T(s_\tau,a_\tau) ),
    \label{eq:lemlinearpolicyvaluefct:1}
\end{align}
where sequential joint policy $\pi$ is a sequence of sequential decision rules $(a_\tau)_{\tau\in \mathbb{N}_{\leq \ell'-1}}$.
By the induction hypothesis, we know state-value function $\upsilon_{M',\lambda,\tau+1}^\pi\colon S\to \mathbb{R}$ under sequential joint policy $\pi$ is linear \wrt sequential occupancy states.
Along with lemmas \ref{lem:linear:transition} and \ref{lem:linear:reward}, Equation (\ref{eq:lemlinearpolicyvaluefct:1}) describes a function that is a composition of linear functions $\upsilon_{M',\lambda,\tau+1}^\pi\colon S\to \mathbb{R}$, $R\colon S\times A\to \mathbb{R}$, and $T\colon S\times A\to S$ \wrt sequential occupancy states---thus a linear function itself \wrt sequential occupancy states.
\end{proof}

\subsection{Linearity of Action-Value Function Under a Policy}
\label{prooflemlinearpolicyvalueqfct}

\begin{restatable}[]{lemma}{lemlinearpolicyvalueqfct}
\label{lem:linear:policyvalue:qfct}
For any sequential occupancy Markov decision process $M'$, action-value function $q_{M',\lambda,0:\ell'}^\pi\colon S\times A\to \mathbb{R}^{\ell'+1}$ is bilinear \wrt sequential occupancy states and stochastic sequential decision rules.
Furthermore, there exists an equivalent action-value function $Q_{M',\lambda,0:\ell'}^\pi \colon \hat{S}\odot\hat{A} \to \mathbb{R}^{\ell'+1}$---\ie 
$
    Q_{M',\lambda,\tau}^\pi\colon
    \hat{s}_\tau\odot\hat{a}_\tau
    \mapsto \hat{r}(\hat{s}_\tau\odot\hat{a}_\tau) 
    + 
    \upsilon_{M',\lambda,\tau+1}^\pi  
    (\hat{p}(\hat{s}_\tau\odot\hat{a}_\tau))
$
which is linear in state-action occupancy measures, with boundary condition 
$
    Q_{M',\lambda,\ell'}^\pi\colon
    \hat{s}_\tau\odot\hat{a}_\tau
    \mapsto 0$.
\end{restatable}
\begin{proof}
The proof is twofold. 
We first show $q_{M',\lambda,0:\ell'}^\pi\colon S\times A\to \mathbb{R}^{\ell'+1}$ is bilinear \wrt sequential occupancy states and stochastic sequential decision rules.
The application of \citeauthor{bellman} equations gives:  for every point in time $\tau$, sequential occupancy state $s_\tau$ and sequential decision rule $a_\tau$, the action value under sequential joint policy $\pi \doteq (a_\tau)_{\tau\in \mathbb{N}_{\leq \ell'-1}}$, 
\begin{align}
    q_{M',\lambda,0:\ell'}^\pi(s_\tau,a_\tau) &=
    R(s_\tau,a_\tau) +
    \upsilon_{M',\lambda,\tau+1}^\pi(T(s_\tau,a_\tau)),
\end{align}
with boundary condition $q_{M',\lambda,0:\ell'}^\pi(s_\tau,a_\tau) \doteq 0$.
A linear transformation of a bilinear function being bilinear, we know $\upsilon_{M',\lambda,\tau+1}^\pi \circ T \colon (s_\tau,a_\tau) \mapsto \upsilon_{M',\lambda,\tau+1}^\pi(T(s_\tau,a_\tau))$ is bilinear \wrt sequential occupancy states and sequential decision rules.
In addition, the addition of bilinear functions being bilinear, we conclude $q_{M',\lambda,\tau}^\pi\colon S\times A\to \mathbb{R}$ is bilinear \wrt sequential occupancy states and stochastic sequential decision rules.
As for the second part of the statement, \ie linearity of $Q_{M',\lambda,0:\ell'}^\pi \colon \hat{S}\odot\hat{A} \to \mathbb{R}^{\ell'+1}$ \wrt sequential state-action occupancy measures, the proof follows directly by application of the composition of linear functions.
Since $Q^\pi_{M',\lambda,\tau}\colon \hat{S}\odot\hat{A} \to \mathbb{R}$ is a composition of linear functions (\ie $\hat{r}$, $\hat{p}$ and $\upsilon_{M',\lambda,\tau+1}^\pi$) \wrt sequential state-action occupancy measures, it is itself a linear function of state-action occupancy measures.
\end{proof}

\subsection{Proof of Theorem \ref{thm:pwlc}}
\label{proofthmpwlc}

\thmpwlc*
\begin{proof}
The proof is twofold. 
The first part shows the piecewise linearity and convexity of $\upsilon_{M',\lambda,0:\ell'}^*\colon S\to \mathbb{R}^{\ell'+1}$. 
To do so, we start with the definition of an optimal state-value function, \ie for every point in time $\tau\in \mathbb{N}_{\leq \ell'-1}$, and sequential occupancy state $s_\tau$, 
\begin{align}
    \upsilon_{M',\lambda,\tau}^*(s_\tau) &\doteq 
    {\textstyle \max_{a_{\tau:\ell'-1}}}~
    \upsilon_{M',\lambda,\tau}^{a_{\tau:\ell'-1}}(s_\tau).
    \label{eq:thmpwlc:1}
\end{align}
By the property that any Markov decision process admits an optimal deterministic policy \citep{bellman,puterman}, one can restrict (stochastic) sequential joint policies $a_{\tau:\ell'-1}$ to be deterministic ones $a^{\mathtt{det}}_{\tau:\ell'-1}$ while still preserving optimality, \ie for every sequential occupancy state $s_\tau$,
\begin{align}
    \upsilon_{M',\lambda,\tau}^*(s_\tau) &= 
    {\textstyle \max_{a^{\mathtt{det}}_{\tau:\ell'-1}}}~
    \upsilon_{M',\lambda,\tau}^{a^{\mathtt{det}}_{\tau:\ell'-1}}(s_\tau).
    \label{eq:thmpwlc:2}
\end{align}
Since the number of deterministic sequential joint policies $a^{\mathtt{det}}_{\tau:\ell'-1}$ is finite, the number of state-value functions $\upsilon_{M',\lambda,\tau}^{a^{\mathtt{det}}_{\tau:\ell'-1}}$ is also finite.
Besides, we know from Lemma \ref{lem:linear:policy:value:fct} (part 1) that state-value function under any arbitrary sequential joint policy is a linear function of sequential occupancy states. 
As a consequence, optimal state-value function $\upsilon_{M',\lambda,0:\ell'}^*\colon S\to \mathbb{R}^{\ell'+1}$ is the maximum over a finite number of linear functions of sequential occupancy states.
Hence $\upsilon_{M',\lambda,0:\ell'}^*\colon S\to \mathbb{R}^{\ell'+1}$ is a piecewise linear and convex function of sequential occupancy states.
The proof of the second part of the statement follows a similar argument.
Once again, we start with the definition of an optimal action-value function, \ie for every point in time $\tau\in \mathbb{N}_{\leq \ell'-1}$, and sequential state-action occupancy measure $\hat{s}_\tau\odot \hat{a}_\tau$, 
\begin{align}
    Q_{M',\lambda,\tau}^*(\hat{s}_\tau\odot \hat{a}_\tau) &\doteq 
    {\textstyle \max_{a_{\tau+1:\ell'-1}}}~
    Q_{M',\lambda,\tau}^{a_{\tau:\ell'-1}}(\hat{s}_\tau\odot \hat{a}_\tau).
    \label{eq:thmpwlc:3}
\end{align}
Again, any Markov decision process admits an optimal deterministic policy \citep{bellman,puterman}. 
Thus, one can restrict (stochastic) sequential joint policies $a_{\tau+1:\ell'-1}$ to deterministic ones $a^{\mathtt{det}}_{\tau+1:\ell'-1}$, while still preserving optimality, \ie for sequential every state-action occupancy measure $\hat{s}_\tau\odot \hat{a}_\tau$,
\begin{align}
    Q_{M',\lambda,\tau}^*(\hat{s}_\tau\odot \hat{a}_\tau) &\doteq 
    {\textstyle \max_{a^{\mathtt{det}}_{\tau+1:\ell'-1}}}~
    Q_{M',\lambda,\tau}^{a^{\mathtt{det}}_{\tau:\ell'-1}}(\hat{s}_\tau\odot \hat{a}_\tau).
    \label{eq:thmpwlc:4}
\end{align}
From Equation (\ref{eq:thmpwlc:4}), we know optimal state-action value function $Q_{M',\lambda,\tau}^*\colon \hat{S}\odot\hat{A}\to \mathbb{R}$
is the maximum over a finite number of linear functions of sequential state-action occupancy measures at point in time $\tau\in \mathbb{N}_{\leq \ell'-1}$, as demonstrated in Lemma \ref{lem:linear:policy:value:fct} (part 2).
As a consequence, optimal state-action value function $Q_{M',\lambda,\tau}^*\colon \hat{S}\odot\hat{A}\to \mathbb{R}$ is a piecewise linear and convex function of sequential state-action occupancy measures.
So is optimal action-value function $Q_{M',\lambda,0:\ell'}^*\colon \hat{S}\odot\hat{A}\to \mathbb{R}^{\ell'+1}$.
\end{proof}

\section{Complexity Analysis}

\subsection{Complexity Analysis For Sequential-Move Cases}
\label{prooflempessimisticupdaterulecomplexity}

\begin{restatable}[]{lem}{lempessimisticupdaterulecomplexity}
\label{lem:pessimistic:update:rule:complexity}
For any soMDP $M'$, let $V_{M',\lambda, \tau+1}$ be a finite collection of linear functions \wrt SOCs providing the max-plane representation of state-value function $\upsilon_{M',\lambda, \tau+1}$. If we let $U^* \in \argmax_{U(\tau)\colon \tau\in \mathbb{N}_{\leq \ell'}} |U(\tau)|$, then
the complexity of performing a pessimistic update illustrated in Theorem \ref{thm:lower:bound} is about $\pmb{O}(|V_{M',\lambda, \tau+1}||X|^2(|U^*||Z|)^{\tau})$.
\end{restatable}
\begin{proof}
Let $\pmb{C}(V_{M',\lambda, \tau+1})$ be the complexity of performing a pessimistic update illustrated in Theorem \ref{thm:lower:bound} using $V_{M',\lambda, \tau+1}$. Let  $\pmb{C}(\alpha_{\tau+1}^{(\kappa)})$ the complexity of performing a pessimistic update illustrated in Theorem \ref{thm:lower:bound} using a singleton $\{\alpha_{\tau+1}^{(\kappa)}\}$. It follows that 
\begin{align*}
    \pmb{C}(V_{M',\lambda, \tau+1}) &=\textstyle \sum_{\alpha_{\tau+1}^{(\kappa)}\colon \kappa\in \mathbb{N}_{\leq k}} \pmb{C}(\alpha_{\tau+1}^{(\kappa)}).
\end{align*} 
Performing a pessimistic update illustrated in Theorem \ref{thm:lower:bound} using a singleton $\{\alpha_{\tau+1}^{(\kappa)}\}$ consists of computing action-value function $\beta_\tau^{(\kappa)}$ in about $\pmb{O}(|X|^2|O||Z||U(\tau)|)$ time complexity. If we let $U^* \in \argmax_{U(\tau)\colon \tau\in \mathbb{N}_{\leq \ell'}} |U(\tau)|$, then $|O|$ is upper-bounded by $\pmb{O}(|U^*|^{\tau-1}|Z|^{\tau-1})$, then $\pmb{C}(\alpha_{\tau+1}^{(\kappa)}) = \pmb{O}(|X|^2|U^*|^{\tau}|Z|^{\tau})$. Consequently, the complexity of performing a pessimistic update illustrated in Theorem \ref{thm:lower:bound} using $V_{M',\lambda, \tau+1}$ is about $\pmb{O}(|V_{M',\lambda, \tau+1}||X|^2|U^*|^{\tau}|Z|^{\tau})$.
Which ends the proof.
\end{proof}

\subsection{Complexity Analysis For Simultaneous-Move Cases}
\label{prooflempessimisticupdaterulecomplexitysimultaneous}

This section establishes the time complexity of performing a backup in the simultaneous-move setting. Before proceeding any further, we first state the piecewise linearity and convexity property of the optimal value functions in the simultaneous-move setting. 

\begin{restatable}[]{thm}{thmconvexity}[Adapted from \citet{Dibangoye:OMDP:2016}]
For any oMDP $\mathtt{M}'$, let $\upsilon^*_{\mathtt{M}', \gamma, 0:\ell}$ be the optimal state-value functions \wrt OCs. Then, for every $t\in \mathbb{N}_{\leq \ell-1}$, there exists a finite collection $V_{\mathtt{M}', \gamma, t} \doteq \{\alpha_t^{(\kappa)}\colon \mathtt{X}\times \mathtt{O}_t \to \mathbb{R} | \kappa \in \mathbb{N}_{\leq k}\}$ of linear functions \wrt OCs, such that: $\upsilon^*_{\mathtt{M}', \gamma, t}\colon \mathtt{s}_t\mapsto \max_{\kappa \in \mathbb{N}_{\leq k}} \langle \mathtt{s}_t, \alpha_t^{(\kappa)}\rangle$, where $\langle \mathtt{s}_t, \alpha_t^{(\kappa)}\rangle \doteq \sum_{\mathtt{x}\in \mathtt{X}}\sum_{\mathtt{o}\in \mathtt{O}_t} \alpha_t^{(\kappa)}(\mathtt{x}, \mathtt{o})\cdot \mathtt{s}_t(\mathtt{x}, \mathtt{o})$.
\end{restatable}

With this property as a background, one can easily express the worst-case time complexity of a point-based backup using the max-plane representation of the value functions \wrt OCs.

\begin{restatable}[]{lem}{lempessimisticupdaterulecomplexitysimultaneous}
\label{lem:pessimistic:update:rule:complexity:simultaneous}
For any oMDP $\mathtt{M}'$, let $V_{\mathtt{M}',\gamma, t+1}$ be a finite collection of linear functions \wrt OCs providing the max-plane representation of state-value function $\upsilon_{\mathtt{M}',\gamma, t+1}$. If we let $\mathtt{U}^* \in \argmax_{\mathtt{U}^i\colon i \in \mathbb{N}_{\leq n}} |\mathtt{U}^i|$ and $\mathtt{Z}^* \in \argmax_{\mathtt{Z}^i\colon i \in \mathbb{N}_{\leq n}} |\mathtt{Z}^i|$, then
the complexity of performing a pessimistic update is about $\pmb{O}(|V_{\mathtt{M}',\gamma, t+1}||\mathtt{X}|^2(|\mathtt{U}^*||\mathtt{Z}^*|)^{nt}|\mathtt{U}^*|^{(|\mathtt{U}^*||\mathtt{Z}^*|)^{nt}} )$.
\end{restatable}
\begin{proof}
    The proof proceeds similarly to that of Lemma \ref{lem:pessimistic:update:rule:complexity} except that we are now in the simultaneous-move case. Unlike the sequential-move case, in the simultaneous-move case, we cannot select the action for each agent based solely on the local information. Instead, we need to enumerate all possible joint decision rules, then evaluate each of them, before selecting the best one. The number of joint decision rules is the size of the set of all mappings from agent histories to actions. If we let $\mathtt{U}^* \in \argmax_{\mathtt{U}^i\colon i \in \mathbb{N}_{\leq n}} |\mathtt{U}^i|$ and $\mathtt{Z}^* \in \argmax_{\mathtt{Z}^i\colon i \in \mathbb{N}_{\leq n}} |\mathtt{Z}^i|$, then the number of joint decision rules is about $\pmb{O}(|\mathtt{U}^*|^{(|\mathtt{U}^*||\mathtt{Z}^*|)^{nt}})$. For each linear value function in $V_{\mathtt{M}',\gamma, t+1}$, we need to select the best joint decision rule. The evaluation complexity of a single decision rule is about $\pmb{O}(|\mathtt{X}|^2(|\mathtt{U}^*||\mathtt{Z}^*|)^{nt})$. So the overall, complexity of the update rule in the simultaneous-move case is about $\pmb{O}(|V_{\mathtt{M}',\gamma, t+1}||\mathtt{X}|^2(|\mathtt{U}^*||\mathtt{Z}^*|)^{nt}|\mathtt{U}^*|^{(|\mathtt{U}^*||\mathtt{Z}^*|)^{nt}} )$. Which ends the proof.
\end{proof}

\section{Exploiting PWLC Lower Bounds}

\subsection{Proof of Corollary \ref{cor:lower:bound}}
\label{proofcorlowerbound}

\corlowerbound*
\begin{proof}
The proof follows directly from the definition of a piecewise linear and convex function along with Theorem \ref{thm:pwlc}, which shows piecewise linearity and convexity of optimal state- and action-value functions of sequential occupancy Markov decision process $M'$.
\end{proof}

\subsection{Proof of Theorem \ref{thm:lower:bound}}
\label{proofthmlowerbound}

\thmlowerbound*
\begin{proof}
The proof builds upon \citeauthor{bellman}'s optimality equations, \ie for any point in time $\tau$, and sequential occupancy state $s_\tau$, 
\begin{align}
    v^{\mathtt{greedy}}_\tau &= 
    \max_{a_\tau}~ 
    \left( 
    r(s_\tau,a_\tau) +
    \upsilon_{M',\lambda,\tau+1}(p(s_\tau,a_\tau))
    \right)
\end{align}
The application of Corollary \ref{cor:lower:bound} onto $\upsilon_{M',\lambda,\tau+1}$, we obtain:
\begin{align}
    v^{\mathtt{greedy}}_\tau &= 
    \max_{a_\tau}~ 
    \left( 
    r(s_\tau,a_\tau) +
    \max_{\kappa\in \mathbb{N}_{\leq k}}~
    \left\langle p(s_\tau,a_\tau), \alpha^{(\kappa)}_{\tau+1} \right\rangle 
    \right)
\end{align}
Re-arranging terms in the above equation leads to the following expression, \ie 
\begin{align}
    v^{\mathtt{greedy}}_\tau &= 
    \max_{a_\tau}~ 
    \max_{\kappa\in \mathbb{N}_{\leq k}}~
    \left( 
    r(s_\tau,a_\tau) +
    \left\langle p(s_\tau,a_\tau), \alpha^{(\kappa)}_{\tau+1} \right\rangle 
    \right)
\end{align}
Exploiting the bilinearity of both reward and transition functions $r$ and $p$ results in the following expression, \ie %
\begin{align}
    v^{\mathtt{greedy}}_\tau &= 
    \max_{a_\tau}~ 
    \max_{\kappa\in \mathbb{N}_{\leq k}}~
    \sum_{x_\tau,o_\tau, u_\tau}
    (\hat{s}_\tau\odot\hat{a}_\tau)(x_\tau,o_\tau, u_\tau)  \\ 
    &\quad 
    \cdot \left( 
    \lambda(\tau) \cdot
    r_{x_\tau,u_\tau}(\tau) +
    \sum_{x_{\tau+1},z_{\tau+1}}~
    p_{x_\tau,x_{\tau+1}}^{u_\tau,z_{\tau+1}}(\tau) \cdot 
    \alpha^{(\kappa)}_{\tau+1}(x_{\tau+1},\langle o_\tau, u_\tau, z_{\tau+1} \rangle)
    \right)
\end{align}
The substitution of quantity in the bracket of the hand-right side of the above equation by $\beta^{(\kappa)}_\tau(x_\tau, o_\tau,u_\tau )$ gives,
\begin{align}
    v^{\mathtt{greedy}}_\tau &= 
    \max_{a_\tau}~ 
    \max_{\kappa\in \mathbb{N}_{\leq k}}~
    \sum_{x_\tau,o_\tau, u_\tau}
    (\hat{s}_\tau\odot\hat{a}_\tau)(x_\tau,o_\tau, u_\tau) \cdot
    \beta^{(\kappa)}_\tau(x_\tau, o_\tau,u_\tau )
\end{align}
Again re-arranging terms to isolate histories $o_\tau^{\rho(\tau)}$ from others yields
\begin{align}
    v^{\mathtt{greedy}}_\tau &= 
    \max_{a_\tau}~ 
    \max_{\kappa\in \mathbb{N}_{\leq k}}~
    \sum_{o^{\rho(\tau)}_\tau}
    \left(
    \sum_{x_\tau,o^{-\rho(\tau)}_\tau}
    s_\tau(x_\tau,o_\tau) \cdot 
    \beta^{(\kappa)}_\tau(x_\tau, o_\tau,a_\tau(o^{\rho(\tau)}_\tau) )
    \right)
\end{align}
It is worth noticing that the value in the bracket quantity of the right-hand side in the above equation depends on $a_\tau$ only through $\kappa$ and $o^{\rho(\tau)}_\tau$, \ie $a^{(\kappa)}_\tau(o^{\rho(\tau)}_\tau)$.
Which means one can move the choice of $u_\tau = a^{(\kappa)}_\tau(o^{\rho(\tau)}_\tau)$ inside the second $\max$ operator of the above equation, \ie 
\begin{align}
    v^{\mathtt{greedy}}_\tau &= 
    \max_{\kappa\in \mathbb{N}_{\leq k}}~
    \sum_{x_\tau,o_\tau}
    s_\tau(x_\tau,o_\tau) \cdot 
    \beta^{(\kappa)}_\tau(x_\tau, o_\tau,a^{(\kappa)}_\tau(o^{\rho(\tau)}_\tau) )\\
    a^{(\kappa)}_\tau(o^{\rho(\tau)}_\tau) &\in 
    \argmax_{u_\tau}~ 
    \sum_{x_\tau,o^{-\rho(\tau)}_\tau}
    s_\tau(x_\tau,o_\tau) \cdot 
    \beta^{(\kappa)}_\tau(x_\tau, o_\tau,u_\tau ).
\end{align}
The substitution of $\beta^{(\kappa)}_\tau(x_\tau, o_\tau,u_\tau )$ by $\alpha_\tau^{(\kappa),a^{(\kappa)}_\tau}$ directly leads to the following results
\begin{align}
    (a^{(\kappa)}_\tau,\alpha^{\mathtt{greedy}}_\tau) &= 
    \argmax_{(a^{(\kappa)}_\tau,\alpha_\tau^{(\kappa),a^{(\kappa)}_\tau})\colon \kappa\in \mathbb{N}_{\leq k}}~
    \left\langle s_\tau, 
    \alpha_\tau^{(\kappa),a^{(\kappa)}_\tau} \right\rangle\\
    a^{(\kappa)}_\tau(o^{\rho(\tau)}_\tau) &\in 
    \argmax_{u_\tau}~ 
    \sum_{x_\tau,o^{-\rho(\tau)}_\tau}
    s_\tau(x_\tau,o_\tau) \cdot 
    \beta^{(\kappa)}_\tau(x_\tau, o_\tau,u_\tau ).
\end{align}
Pair $(a^{(\kappa)}_\tau,\alpha^{\mathtt{greedy}}_\tau)$ is guaranteed to achieve performance $v^{\mathtt{greedy}}_\tau$ at sequential occupancy state $s_\tau$ and finite collection $V_{M',\lambda,\tau}$. 
Which ends the proof.
\end{proof}

\section{The $o$SARSA Algorithm}
 \label{sec:alg:osarsa}

\begin{algorithm}[!h]
    \caption{The $\epsilon$-greedy decision-rule selection with a portfolio of heuristics.}
    \label{algo:oSARSA_portfolio}
      {
      	    Set $\epsilon$.\\
	    Set current state $s_\tau$.\\
	    Set next-step state-value linear function $\alpha_{\tau+1}$.\\
	    Set the portfolio, \eg $\Pi \gets \{\pi^{\mathtt{random}}, \pi^{\mathtt{mdp}},  \pi^{\mathtt{blind}}\}$. \\
	    Set the distribution over $\Pi$, \eg $\nu(\Pi) \gets \{0.5.,0.25,0.25\}$.\\
     	    \BlankLine
    
	    \uIf{$\epsilon < \mathtt{random(0,1)} $}{
        			$a^{s_\tau}_\tau \gets$ greedy decision-rule selection \wrt $\alpha_{\tau+1}$.\\
    	     }
 
	     \uElse{
        		Sample a heuristic policy $\pi\sim \nu$ from the portfolio. \\
        		$a^{s_\tau}_\tau \gets $ extract a valid decision rule from $\pi$ at state $s_\tau$.\\
    	     }   

	    \BlankLine
	    $\pmb{\mathtt{return}}$ the selected decision rule $a^{s_\tau}_\tau$.
     }     
\end{algorithm}

\begin{algorithm}[!h]
    \caption{The acceptance rule based on simulated annealing.}
    \label{algo:oSARSA_SA}
      	    Set $\epsilon$.\\
      	    Set $\varsigma>0$.\\
	    Set current state $s_\tau$.\\
	    Set best value so far $g_{\tau+1}$.\\
	    Set selected decision rule $a^{s_\tau}_\tau$.\\
	    Set next-step state-value linear function $\alpha_{\tau+1}$.\\
     	    \BlankLine

    Compute   $\mathtt{temperature} \gets \varsigma \cdot \epsilon$ and  $\tilde{g}_{\tau+1} \gets \alpha_{\tau+1}(T(s_\tau, a^{s_\tau}_\tau))$. \\
    $\pmb{\mathtt{return}}$ {$\tilde{g}_{\tau+1} \geq g_{\tau+1}$, or $\exp{\left( \frac{\tilde{g}_{\tau+1} - g_{\tau+1}}{ \mathtt{temperature}} \right)} > \mathtt{random(0,1)}$, or $\mathtt{temperature}=0$}. 
\end{algorithm}

Algorithm \ref{algo:oSARSA} describes an enhanced version of the $o$SARSA algorithm \citep{DibangoyeBICML18}.
This version of the $o$SARSA algorithm differs in the exploration strategy, the greedy-action selection rule, and the policy-improvement rule. The original $o$SARSA algorithm used a standard $\epsilon$-greedy exploration strategy, whereas the enhanced version exploits a portfolio of three policies: random policy, underlying MDP policy, and blind policy. Throughout the experiments, we initialize $\epsilon$ to $0.5$, \cf \Cref{algo:oSARSA}. Furthermore, the distribution over the portfolio is set to $\{0.5, 0.25, 0.25\}$, \cf \Cref{algo:oSARSA_portfolio}. In such a setting, the random action selection in the standard $\epsilon$-greedy exploration strategy is replaced by a selection of the next action according to the distribution over the portfolio. The greedy-action selection rule differs according to the single-agent reformulation at hand. We used mixed-integer linear programs to perform the greedy action selection when facing simultaneous-move reformulation $\mathtt{M}'$. The resulting algorithm is noted as $o$SARSA$^\mathtt{sim}$. On the other hand, we used the sequential greedy action selection, \cf Theorem \ref{thm:lower:bound}, when facing sequential-move reformulation $M'$. The resulting algorithm is named $o$SARSA$^\mathtt{seq}$. The last  of the differences with the original $o$SARSA algorithm is the policy-improvement rule. The enhanced version uses a metaheuristic, an acceptance rule based on Simulated Annealing \citep{rutenbar1989simulated}, in the selection of the next best policy. In particular, a modification of the current policy is kept not only if the performance improves but also in the opposite case, with a probability depending on the loss and on a temperature coefficient decreasing with time, \cf \Cref{algo:oSARSA_SA}.  The temperature coefficient was set to $4$ across all experiments. 

\section{Many-Agent Domains} \label{sec:multi:agent:benchmarks}

This section describes the many-agent domains used in our experiments. These domains were previously extended from the two-agent domains in \citet{peralez2024solving}. We report them here for the sake of completeness. We extended another two-agent domain, namely gridsmall, to a many-agent version. 

\subsection{Multi-Agent Tiger}
The two-agent tiger problem describes a scenario where agents face two closed doors, one of which conceals a treasure while the other hides a dangerous tiger. Neither agent knows which door leads to the treasure and which one to the tiger, but they can receive partial and noisy information about the tiger's location by listening. At any given time, each agent can choose to open either the left or right door, which will either reveal the treasure or the tiger, and reset the game. To gain more information about the tiger's location, agents can listen to hear the tiger on the left or right side, but with uncertain accuracy.  We have extended this problem to an $n$-agent version.
Each agent keeps the same individual actions and observations; as a consequence, the total number of actions and observations grows exponentially with $n$: $|U| = 3^n$, $|Z| = 2^n$. The state space, corresponding to the position of the tiger, is unchanged, i.e. $|S| = 2$.
The reward function was adapted. Denoting $r_l$ and $r_g$ the ratio of agents listening and opening the good door respectively, in the absence of wrong openings the reward $r$ is set to $r = -2 \times r_l + 20 \times r_g$. Hence, as in the original $2$-agent game, the maximum reward is $20$ (when all agent open the good door) while the simple policy consisting in listening whatever is observed still accumulates $-2$ at each time-step. In the presence of wrong openings, the reward is set to $r = -2 \times r_l + 20 \times r_g -100 / c $, with $c = 1 + (|wrongDoors| - 1) / (n - 1)$, encouraging agents to synchronize.

\subsection{Many-Agent GridSmall}
In this scenario, $n$ agents want to meet as quickly as possible on a two-dimensional grid. Each agent’s
possible actions include moving north, south, west, east and staying at the same place. The actions of a
given agent do not affect the other agents. After taking an action, each agent can sense partial information about its own location: the agent is informed about its column in the grid but not about its row. The state of the problem is the position of the agents on the $2\times 2$ grid, leading to $|S| = 4^n$. Each agent keeps the same individual actions and observations as in the original $2$-agent setting; in consequence the total number of actions and observations also grows exponentially with $n$: $|U| = 5^n$, $|Z| = 2^n$. When an agent try to move in a direction, its action is misinterpreted with a probability of $0.5$, while after choosing the stay action the agent always stays in the same cell. The reward is set to $1$ when \emph{all} agents are in the same cell and to $0$ otherwise.

\subsection{Many-Agent Recycling Robot}
The recycling robot task is as a single-agent problem. Later on,  \citet{amato2012optimizing} generalized as a two-agent version in 2012. The multi-agent formulation requires robots to work together to recycle soda cans. In this problem, both robots have a battery level, which can be either high or low. They have to choose between collecting small or big wastes and recharging their own battery level. Collecting small or big wastes can decrease the robot's battery level, with a higher probability when collecting the big waste. When a robot's battery is completely exhausted, it needs to be picked up and placed onto a recharging spot, which results in a negative reward. The coordination problem arises as robots cannot pick up a big waste independently. To solve this problem, an $n$-agent Dec-POMDP was derived by allowing a big soda to be picked up only when all robots try to collect it simultaneously. The other transition and observation probabilities come from the $n$ independent single-agent models. 

\subsection{Many-Agent Broadcast Channel}
In 1996, \citet{Ooi96} introduced a scenario in which a unique channel is shared by $n$ agents, who aim at transmitting packets. The time is discretized, and only one packet can be transmitted at each time step. If two or more agents attempt to send a packet at the same time, the transmission fails due to a collision. In 2004, \citet{HansenBZ04} extended this problem to a partially observable one, focusing on two agents. We used similar adaptations to define a partially observable version of the original $n$-agent broadcast channel.

\section{Experiments}



\subsection{Scalability}

This section studies how do the increase in the planning horizons or team sizes affect the performance of both $o$SARSA$^\mathtt{seq}$ and $o$SARSA$^\mathtt{sim}$ on two-agent as well as many-agent domains.

\paragraph{Planning Horizons.} Figure \ref{fig:iter_MarsGrid3x3} shows  $o$SARSA$^\mathtt{seq}$ outperforms $o$SARSA$^\mathtt{sim}$ for different planning horizons on two-agent mars and grid3x3. The complexity of a single episode grows exponentially using $o$SARSA$^\mathtt{sim}$, whereas it grows polynomially using $o$SARSA$^\mathtt{seq}$. Perhaps surprisingly, $o$SARSA$^\mathtt{seq}$ does not suffer from increasing the planning horizon. This was unexpected since the planning horizon $\ell'$ in the sequential central planning reformulation is $n$ times longer than the planning horizon $\ell$ in the simultaneous central planning reformulation. In general, for small team sizes, the increase in the planning  horizon does not seem to be a major bottleneck in the scalability of the $o$SARSA$^\mathtt{seq}$ algorithm.

\begin{figure}[!h]	
    \caption{Average iteration time for two-agent mars and grid3x3 problems with different planning horizons. Note that for the larger planning horizon ($\ell=100$) on grid3x3, $o$SARSA$^\mathtt{sim}$ is unable to complete even one iteration within the one-hour period.}
    \label{fig:iter_MarsGrid3x3}
    \centering
    \includegraphics[width=.5\columnwidth]{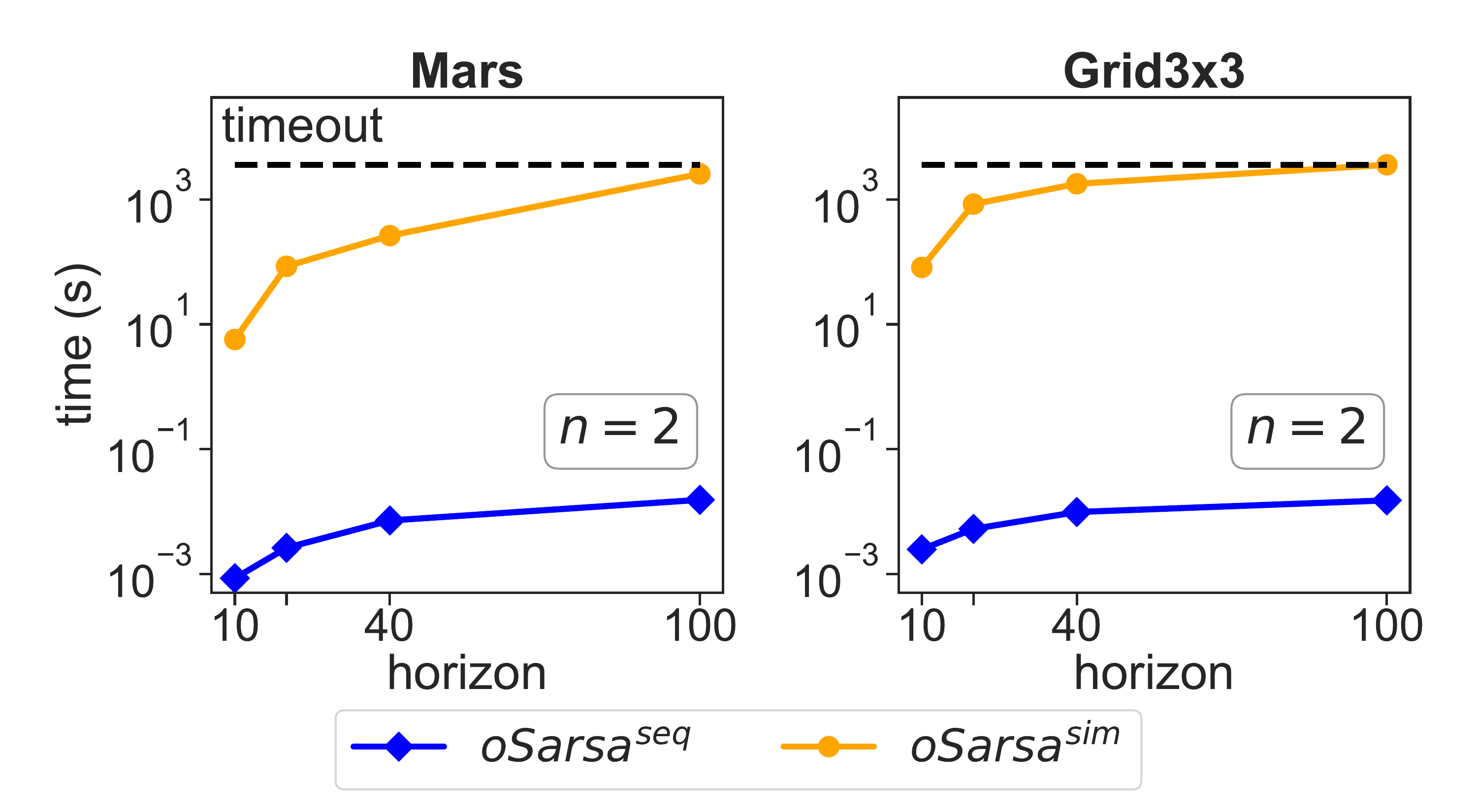}
\end{figure}

\paragraph{Team Sizes.} Figure \ref{fig:iter_TigerRecycling}  shows  $o$SARSA$^\mathtt{seq}$ outperforms $o$SARSA$^\mathtt{sim}$ for different team sizes on recycling and tiger for a fixed (simultaneous-move) planning horizon $\ell=10$. It appears that $o$SARSA$^\mathtt{sim}$ quickly runs out of time as the team size increases. For the many-agent tiger problem, $o$SARSA$^\mathtt{sim}$ runs out of time at $n=3$ whereas $o$SARSA$^\mathtt{seq}$ runs out of time at $n=6$. The reason why the increase in the team size does still somehow affect $o$SARSA$^\mathtt{seq}$ is mainly because state, action and observation spaces increase exponentially with increasing $n$ in the many-agent tiger and recycling problems.  Consequently, even loading these many-agent problems become cumbersome, let alone optimally solving them. There are two main avenues to address this limitation. Factored problems enable us to deal with larger teams while preserving manageable state, action and observation spaces. A notable example includes the network distributed Markov decision process framework \citep{Jillesaamas14}. The second avenue pertains to searching for approximate solutions instead of optimal ones. Neural approaches could help addressing the exponential increase in state, action and observation spaces, yet at the expense of losing the theoretical guarantees.

\begin{figure}[!h]	
    \caption{Average iteration time for many-agent recycling and tiger problems with different team sizes.}
    \label{fig:iter_TigerRecycling}
    \centering
    \includegraphics[width=.5\columnwidth]{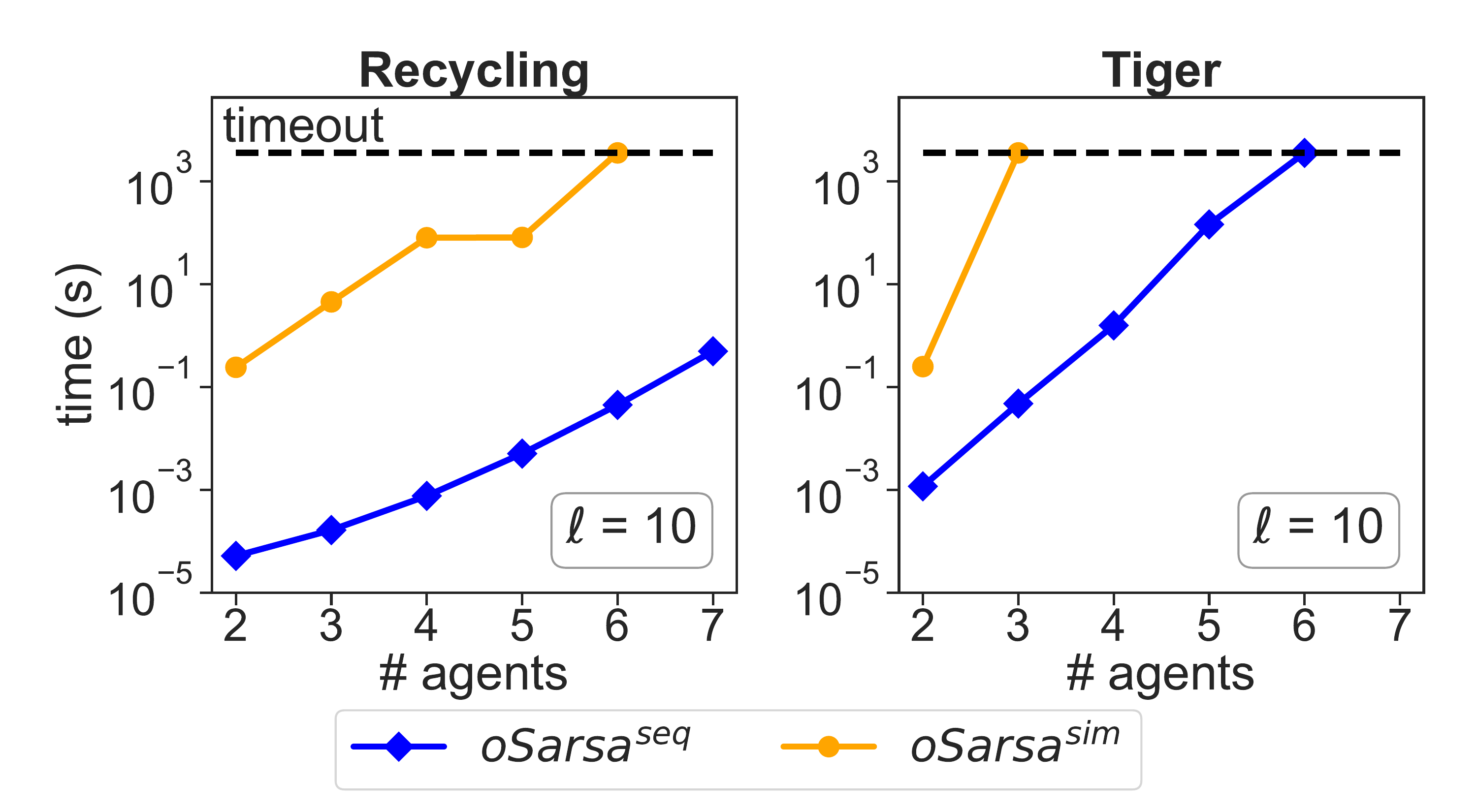}
\end{figure}

\subsection{Anytime Performances}

This section studies the speed of convergence for both $o$SARSA$^\mathtt{seq}$ and $o$SARSA$^\mathtt{sim}$ on two-agent as well as many-agent domains. For the sake of completeness, we also reported the performance of the state-of-the-art algorithm for Dec-POMDPs, FB-HSVI, as well as reinforcement learning approaches, A2C and IQL.

\paragraph{Two-Agent Domains.} Figure \ref{fig:anytime_mars} shows both $o$SARSA$^\mathtt{seq}$ and $o$SARSA$^\mathtt{sim}$ converge towards the best-known values for the two-agent mars problem at planning horizons $\ell=20$ and $\ell=40$. It is worth noticing that at planning horizons $\ell=20$ and $\ell=40$, $o$SARSA$^\mathtt{seq}$ exhibits a faster convergence rate than $o$SARSA$^\mathtt{sim}$. Non-surprisingly, reinforcement learning approaches, A2C and IQL, stabilized at local optima, far away from the best-known values. This is not surprising because these algorithms are not geared to find an optimal solution, unless the problem is weakly coupled, \eg two-agent recycling problem. In contrast, $o$SARSA$^\mathtt{seq}$ and $o$SARSA$^\mathtt{sim}$ are global methods capable of finding optimal solution asymptotically. When the available resources are enough, \eg time budget, they will eventually find an optimal solution. When, instead, the resources are scarce, they may fail to find an optimal solution. A surprising effect in the anytime performance is that the modified policy-improvement rule with the metaheuristic does not seem to negatively affect the overall performance of neither $o$SARSA$^\mathtt{seq}$ nor $o$SARSA$^\mathtt{sim}$. To gain a better understanding on this insight, we refer the reader to our ablation study. 

\begin{figure}[!h]	
    \caption{Anytime values for the two-agent mars problem with planning horizons $\ell=20$ and $\ell=40$}.
    \label{fig:anytime_mars}
    \centering
    \includegraphics[width=.5\columnwidth]{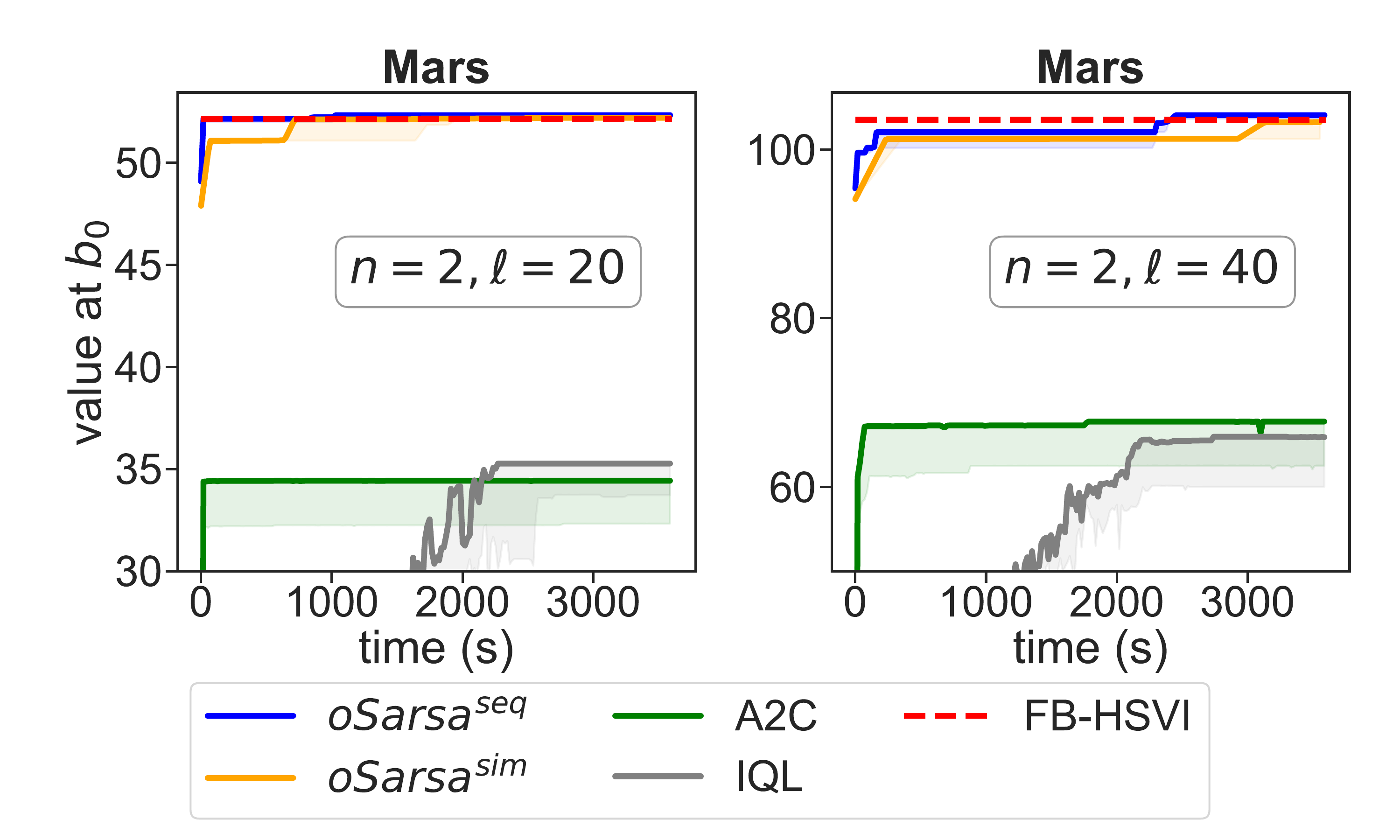}
\end{figure}

\paragraph{Many-Agent Domains.} Figure \ref{fig:anytime_recycling} shows both  $o$SARSA$^\mathtt{seq}$ and $o$SARSA$^\mathtt{sim}$ quickly converge toward the best-known values for the two-agent recycling problem at planning horizon $\ell=10$ and team size $n=2$.  Unsurprisingly, reinforcement learning approaches, A2C and IQL, stabilized at local optima, far away from the best-known values for the two-agent recycling problem at planning horizon $\ell=10$ and team sizes $n=2$ and $n=4$. $o$SARSA$^\mathtt{sim}$ fails to converge to the best-known value for the two-agent recycling problem at planning horizon $\ell=10$ and team size $n=4$. The main reason for this is the scarce resource available for this task. Indeed, as the team size increases, performing even a single episode becomes cumbersome for $o$SARSA$^\mathtt{sim}$.
In constast, $o$SARSA$^\mathtt{seq}$ quickly converges to the best-known value for the two-agent recycling problem at planning horizon $\ell=10$ and team size $n=4$. This is mainly because the sequential backups are polynomial instead of double exponential. So, performing a single episode when using  $o$SARSA$^\mathtt{seq}$ is not prohibitive in comparison to when using $o$SARSA$^\mathtt{sim}$. 

\begin{figure}[!h]	
    \caption{Anytime values for many-agent recycling problem with $n=2$ and $n=4$ agents.}
    \label{fig:anytime_recycling}
    \centering
    \includegraphics[width=.5\columnwidth]{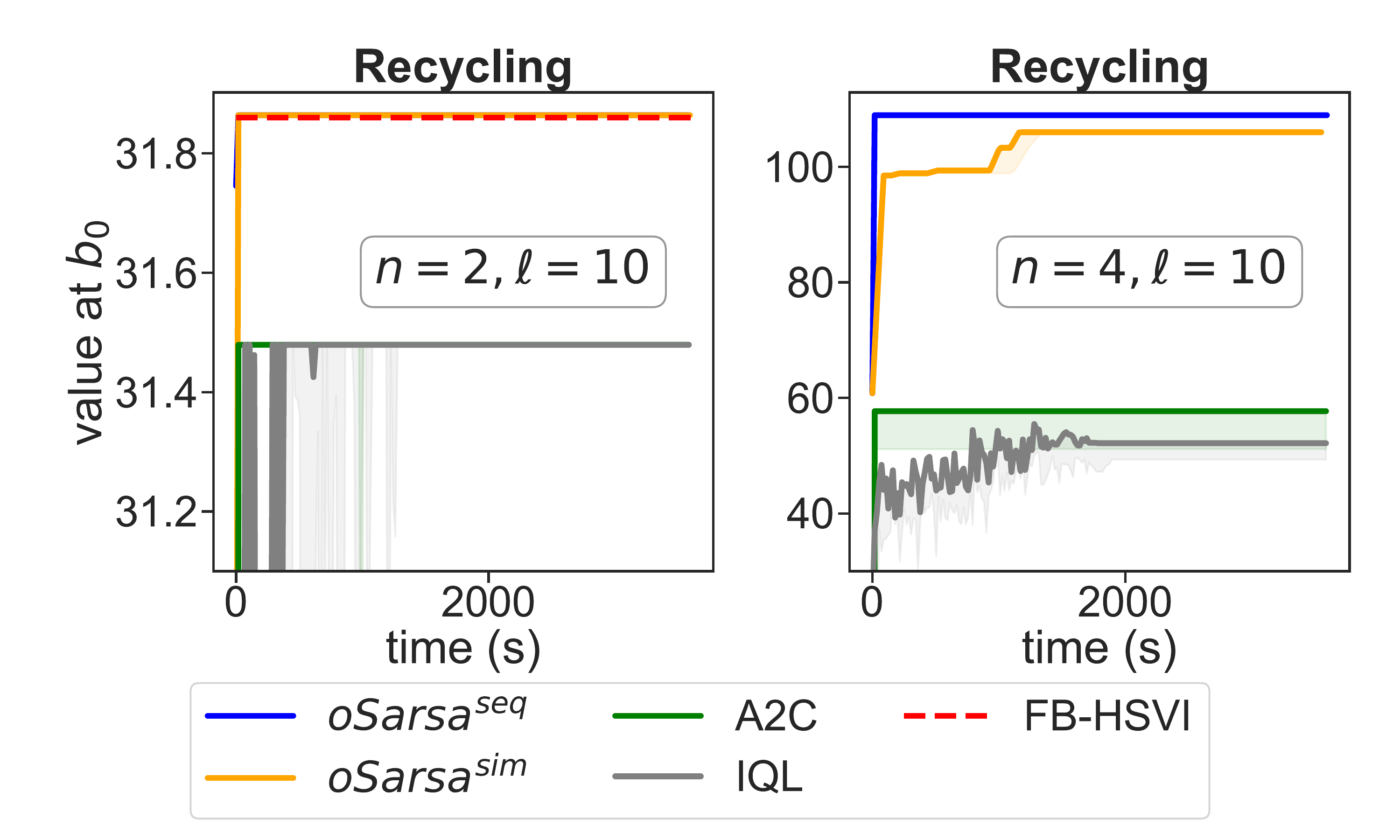}
\end{figure}

\subsection{Ablation study}

We conducted an ablation study to evaluate gains provided by the use of a portfolio of policies and the simulated annealing metaheuristic. We distinguish between two- and many-agent domains, reported in Tables \ref{table:ablation_2agents} and \ref{table:ablation_multiagents}  respectively. We compared $o$SARSA$^\mathtt{seq}$ against three variants: $o$SARSA$^\mathtt{seq}_1$, $o$SARSA$^\mathtt{seq}_2$, and $o$SARSA$^\mathtt{seq}_3$. If we disable from $o$SARSA$^\mathtt{seq}$ the utilisation of the portfolio, it becomes $o$SARSA$^\mathtt{seq}_1$. Instead, if we disable from $o$SARSA$^\mathtt{seq}$ the utilisation of simulated annealing, it becomes $o$SARSA$^\mathtt{seq}_2$. Finally, if we disable from $o$SARSA$^\mathtt{seq}$ the utilisation of both the portfolio and  simulated annealing, it becomes $o$SARSA$^\mathtt{seq}_3$.

\paragraph{Two-Agent Domains.} Table \ref{table:ablation_2agents} shows that overall $o$SARSA$^\mathtt{seq}$ outperforms  $o$SARSA$^\mathtt{seq}_1$, $o$SARSA$^\mathtt{seq}_2$, and $o$SARSA$^\mathtt{seq}_3$ on most tested two-agent domains. It also demonstrates that the utilisation of the portfolio has far positive impact on the performance than the utilisation of simulated annealing. Indeed, the margins between $o$SARSA$^\mathtt{seq}$ and $o$SARSA$^\mathtt{seq}_2$ versus $o$SARSA$^\mathtt{seq}$ and $o$SARSA$^\mathtt{seq}_1$ are largely in favor of $o$SARSA$^\mathtt{seq}_2$. In boxpushing(2) at planning horizon $\ell=100$, for instance, the margin of $o$SARSA$^\mathtt{seq}$  \wrt $o$SARSA$^\mathtt{seq}_1$ is about $1575.53$ while it is only about $5.71$ \wrt $o$SARSA$^\mathtt{seq}_2$. Yet, $o$SARSA$^\mathtt{seq}_1$ outperforms $o$SARSA$^\mathtt{seq}_3$ in most tested two-agent domains. This insight suggests that simulated annealing has a positive effect on the original version of the $o$SARSA algorithm. 


\begin{table}[!h]
 \centering%
\caption{Ablation study for $o$SARSA$^\mathtt{seq}$. Ablations: i) no portfolio ($o$SARSA$^\mathtt{seq}_1$), ii) no simulated annealing ($o$SARSA$^\mathtt{seq}_2$), iii) no portfolio nor simulated annealing ($o$SARSA$^\mathtt{seq}_3$).}
\begin{tabular}{@{}ll r r r r}

\toprule
 \bfseries Domain($n$) & \bfseries planning horizon $\ell$
 & \bfseries $o$SARSA$^\mathtt{seq}$
 & \bfseries $o$SARSA$^\mathtt{seq}_1$
 & \bfseries $o$SARSA$^\mathtt{seq}_2$
 & \bfseries $o$SARSA$^\mathtt{seq}_3$
 \\

\cmidrule[0.4pt](r{0.125em}){1-2}%
\cmidrule[0.4pt](lr{0.125em}){3-3}%
\cmidrule[0.4pt](lr{0.125em}){4-4}%
\cmidrule[0.4pt](l{0.125em}){5-5}%
\cmidrule[0.4pt](l{0.125em}){6-6}%

tiger(2) & $10$ & $\highest{15.18}$ & $\highest{15.18}$ & $\highest{15.18}$ &$13.57$ \\ 
\myrowcolour
&$20$ &$\highest{30.37}$ &$27.14$ &$\highest{30.37}$ &$27.14$\\ 
&$40$  &$\highest{67.09}$ &$65.48$ &$\highest{67.09}$ &$65.48$\\ 
\myrowcolour
&$100$ &$\highest{170.91}$  &$169.30$ &$\highest{170.91}$ &$169.30$\\ 
gridsmall(2) & $10$  &$\highest{6.03}$ &$\highest{6.03}$ &$\highest{6.03}$ &$\highest{6.03}$\\ 
\myrowcolour
&$20$ &$\highest{13.96}$ &$\highest{13.96}$ &$13.90$  &$13.90$\\ 
&$40$ &$\highest{30.93}$  &$\highest{30.93}$ &$29.91$ &$29.91$\\ 
\myrowcolour
&$100$ &$78.37$ &$\highest{82.14}$ &$78.39$ &$78.41$\\ 
boxpushing(2) & $10$ &$\highest{224.26}$ &$199.98$ &$223.57$ &$209.67$\\ 
\myrowcolour
&$20$ &$\highest{470.43}$ &$319.38$ &$466.79$ &$307.87$\\ 
&$40$ &$941.07$ &$512.43$ &$\highest{941.57}$ &$415.91$\\ 
\myrowcolour
&$100$ &$\highest{2366.21}$ &$790.68$ &$2360.50$ &$778.22$ \\
mars(2) & $10$ &$\highest{26.31}$ &$25.74$ &$\highest{26.31}$ &$\highest{26.30}$\\ 
\myrowcolour
&$20$  &$\highest{52.32}$ &$52.19$ &$\highest{52.32}$ &$52.20$\\ 
&$40$  &$\highest{104.07}$ &$101.89$ &$103.57$ &$98.29$ \\ 
\myrowcolour
&$100$  &$255.18$ &$209.14$ &$\highest{256.77}$ &$209.70$ \\ 

\bottomrule
\end{tabular}
\label{table:ablation_2agents}
\end{table}

\paragraph{Many-Agent Domains.} Table \ref{table:ablation_multiagents} shows that overall $o$SARSA$^\mathtt{seq}$ outperforms  $o$SARSA$^\mathtt{seq}_1$, $o$SARSA$^\mathtt{seq}_2$, and $o$SARSA$^\mathtt{seq}_3$ on most tested many-agent domains. It also demonstrates that the utilisation of the portfolio has far positive impact on the performance than the utilisation of simulated annealing. Indeed, the margins between $o$SARSA$^\mathtt{seq}$ and $o$SARSA$^\mathtt{seq}_2$ versus $o$SARSA$^\mathtt{seq}$ and $o$SARSA$^\mathtt{seq}_1$ are largely in favor of $o$SARSA$^\mathtt{seq}_2$. It is also worth noticing that $o$SARSA$^\mathtt{seq}_1$ outperforms $o$SARSA$^\mathtt{seq}_3$ in most tested many-agent domains. This insight suggests that simulated annealing has a positive effect on the original version of the $o$SARSA algorithm. 

\begin{table}[!h]
 \centering%
\caption{Ablation study for the sequential version of $o$SARSA$^\mathtt{seq}$. Ablations: i) no portfolio ($o$SARSA$^\mathtt{seq}_1$), ii) no simulated annealing (oSARSA$^{seq}_2$), iii) no portfolio nor simulated annealing (oSARSA$^{seq}_3$)}
\begin{tabular}{@{}ll r r r r}

\toprule%
 \bfseries Domain($n$) & \bfseries planning horizon $\ell$
 & \bfseries $o$SARSA$^\mathtt{seq}$
 & \bfseries $o$SARSA$^\mathtt{seq}_1$
 & \bfseries $o$SARSA$^\mathtt{seq}_2$
 & \bfseries $o$SARSA$^\mathtt{seq}_3$
 \\

\cmidrule[0.4pt](r{0.125em}){1-2}%
\cmidrule[0.4pt](lr{0.125em}){3-3}%
\cmidrule[0.4pt](lr{0.125em}){4-4}%
\cmidrule[0.4pt](l{0.125em}){5-5}%
\cmidrule[0.4pt](l{0.125em}){6-6}%

tiger(3) & $10$ &$\highest{11.29}$ &$\highest{11.29}$ &$\highest{11.29}$ &$8.17$ \\ 
\myrowcolour
tiger(4) & $10$ &$\highest{6.80}$ &$1.94$ &$\highest{6.80}$ &$1.94$ \\ 
tiger(5) & $10$ &$2.41$ &$-4.15$ &$\highest{3.69}$ &$-4.15$ \\
\myrowcolour
recycling(3) & $10$ &$\highest{85.23}$ &$\highest{85.23}$ &$\highest{85.23}$ &$85.02$ \\ 
recycling(4) & $10$ &$\highest{108.92}$ &$\highest{108.92}$ &$\highest{108.92}$ &$\highest{108.92}$ \\ 
\myrowcolour
recycling(5) & $10$ &$\highest{133.84}$ &$122.50$ &$\highest{133.84}$ &$\highest{133.84}$ \\ 
recycling(6) & $10$ &$\highest{159.00}$ &$112.25$ &$\highest{159.00}$ &$146.70$ \\ 
\myrowcolour
recycling(7) & $10$ &$\highest{185.50}$ &$130.96$ &$\highest{185.50}$ &$130.96$ \\ 
gridsmall(3) & $10$ &$\highest{5.62}$ &$\highest{5.62}$ &$\highest{5.62}$ &$5.17$\\ 
\myrowcolour
gridsmall(4) & $10$ &$\highest{4.09}$ &$4.04$ &$4.04$ &$4.04$\\ 

\bottomrule
\end{tabular}
\label{table:ablation_multiagents}
\end{table}

\end{document}